\documentclass[10pt,twocolumn,letterpaper]{article}

\usepackage{cvpr}              %

\usepackage{graphicx}
\usepackage{amsmath}
\usepackage{amssymb}
\usepackage{booktabs}
\usepackage{balance}
\usepackage{pifont}
\usepackage{enumitem}

\usepackage{ulem}

\usepackage[dvipsnames]{xcolor}
\definecolor{indigo}{RGB}{63, 81, 181} %
\definecolor{red}{RGB}{219, 68, 55} %
\definecolor{pink}{RGB}{253, 30, 98}
\definecolor{green}{RGB}{15, 157, 88}

\usepackage[pagebackref=true,breaklinks=true,letterpaper=true,colorlinks,bookmarks=false]{hyperref}
\hypersetup{
    colorlinks,
    linkcolor={indigo},
    citecolor={indigo},
    urlcolor={indigo}
}

\usepackage{array}
\usepackage{multirow}
\usepackage{multicol}
\usepackage{stfloats} 

\newcolumntype{P}[1]{>{\centering\arraybackslash}p{#1}}

\newcommand{\PreserveBackslash}[1]{\let\temp=\\#1\let\\=\temp}
\newcolumntype{C}[1]{>{\PreserveBackslash\centering}p{#1}}
\newcolumntype{R}[1]{>{\PreserveBackslash\raggedleft}p{#1}}
\newcolumntype{L}[1]{>{\PreserveBackslash\raggedright}p{#1}}

\usepackage{amsmath}
\usepackage{amssymb}
\usepackage{mathtools}
\usepackage{amsthm}
\usepackage{amsfonts}
\usepackage{bm}

\newcommand{\pred}{\bm{f}}
\newcommand{\vx}{\bm{x}}
\newcommand{\vz}{\bm{z}}
\newcommand{\vu}{\bm{u}}
\newcommand{\vv}{\bm{v}}
\newcommand{\vd}{\bm{\delta}}
\newcommand{\vm}{\bm{m}}
\newcommand{\vga}{\bm{\gamma}}

\newcommand{\explainer}{\bm{g}}

\newcommand{\eva}{\text{EVA}}
\newcommand{\evaEmp}{\text{EVA}\textsubscript{emp}\xspace}
\newcommand{\evaH}{\text{EVA}\textsubscript{hybrid}\xspace}
\newcommand{\adv}{\textit{adversarial overlap}\xspace}
\newcommand{\AO}{\textit{AO}\xspace}
\newcommand{\AOup}{\overline{\textit{AO}}\xspace}
\newcommand{\AOemp}{\hat{\textit{AO}}\xspace}
\newcommand{\Adv}{\textit{Adversarial overlap}}

\newcommand{\ball}{\mathcal{B}}
\newcommand{\ballu}{\mathcal{B}_{\bm{u}}}

\newcommand{\rsr}{{Robustness\text{-}S\textsubscript{r}}}

\DeclareMathOperator*{\argmax}{arg\,max}
\DeclareMathOperator*{\argmin}{arg\,min}

\usepackage[capitalize]{cleveref}
\crefname{section}{Sec.}{Secs.}
\Crefname{section}{Section}{Sections}
\Crefname{table}{Table}{Tables}
\crefname{table}{Tab.}{Tabs.}

\theoremstyle{plain}
\newtheorem{theorem}{Theorem}[section]
\newtheorem{proposition}[theorem]{Proposition}

\theoremstyle{definition}
\newtheorem{definition}[theorem]{Definition}

\theoremstyle{remark}

\def\cvprPaperID{*****} %
\def\confName{CVPR}
\def\confYear{2023}

\begin{document}

\title{Don't Lie to Me!\\ Robust and Efficient Explainability with Verified Perturbation Analysis}

\author{
    \hspace{0cm} 
    \textbf{Thomas Fel}$^{1,2,5}$\footnotemark[2]
    \hspace{2mm}
    \textbf{Melanie Ducoffe}$^{2,7}$\footnotemark[2]
    \hspace{2mm}
    \textbf{David Vigouroux}$^{2,4}$\footnotemark[2] 
    \hspace{2mm}
    \textbf{Rémi Cadène}$^{1, 3}$
    \hspace{2mm}
    \\
    \hspace{1cm}
    \textbf{Mikael Capelle}$^{6}$
    \hspace{1cm}
    \textbf{Claire Nicodème}$^{5}$
    \hspace{1cm}
    \textbf{Thomas Serre}$^{1,2}$
    \hspace{0cm}
\\
\\
$^1$Carney Institute for Brain Science, Brown University, USA \\
$^2$Artificial and Natural Intelligence Toulouse Institute \\
$^3$Sorbonne Université, CNRS, France  ~
$^4$IRT Saint-Exupery, France \\
$^5$Innovation \& Research Division, SNCF  ~
$^6$Thales Alenia Space, France \\
$^7$Airbus AI Research
}
\maketitle

\begin{abstract}

A plethora of attribution methods have recently been developed to explain deep neural networks.
These methods use different classes of perturbations (e.g, occlusion, blurring, masking, etc) to estimate the importance of individual image pixels to drive a model's decision.
Nevertheless, the space of possible perturbations is vast and current attribution methods typically require significant computation time to accurately sample the space in order to achieve high-quality explanations. 
In this work, we introduce EVA (Explaining using Verified Perturbation Analysis) -- the first explainability method which comes with guarantees that an entire set of possible perturbations has been exhaustively searched. We leverage recent progress in verified perturbation analysis methods to directly propagate bounds through a neural network to exhaustively probe a -- potentially infinite-size --  set of perturbations in a single forward pass. Our approach takes advantage of the beneficial properties of verified perturbation analysis, i.e., time efficiency and guaranteed complete -- sampling agnostic -- coverage of the perturbation space -- to identify image pixels that drive a model's decision. 
We  evaluate EVA systematically and demonstrate state-of-the-art results on multiple benchmarks. Our code is freely available:
\href{https://github.com/deel-ai/formal-explainability}{\nolinkurl{github.com/deel-ai/formal-explainability}}

\end{abstract}

\begin{figure}[t!]
  \includegraphics[width=0.5\textwidth]{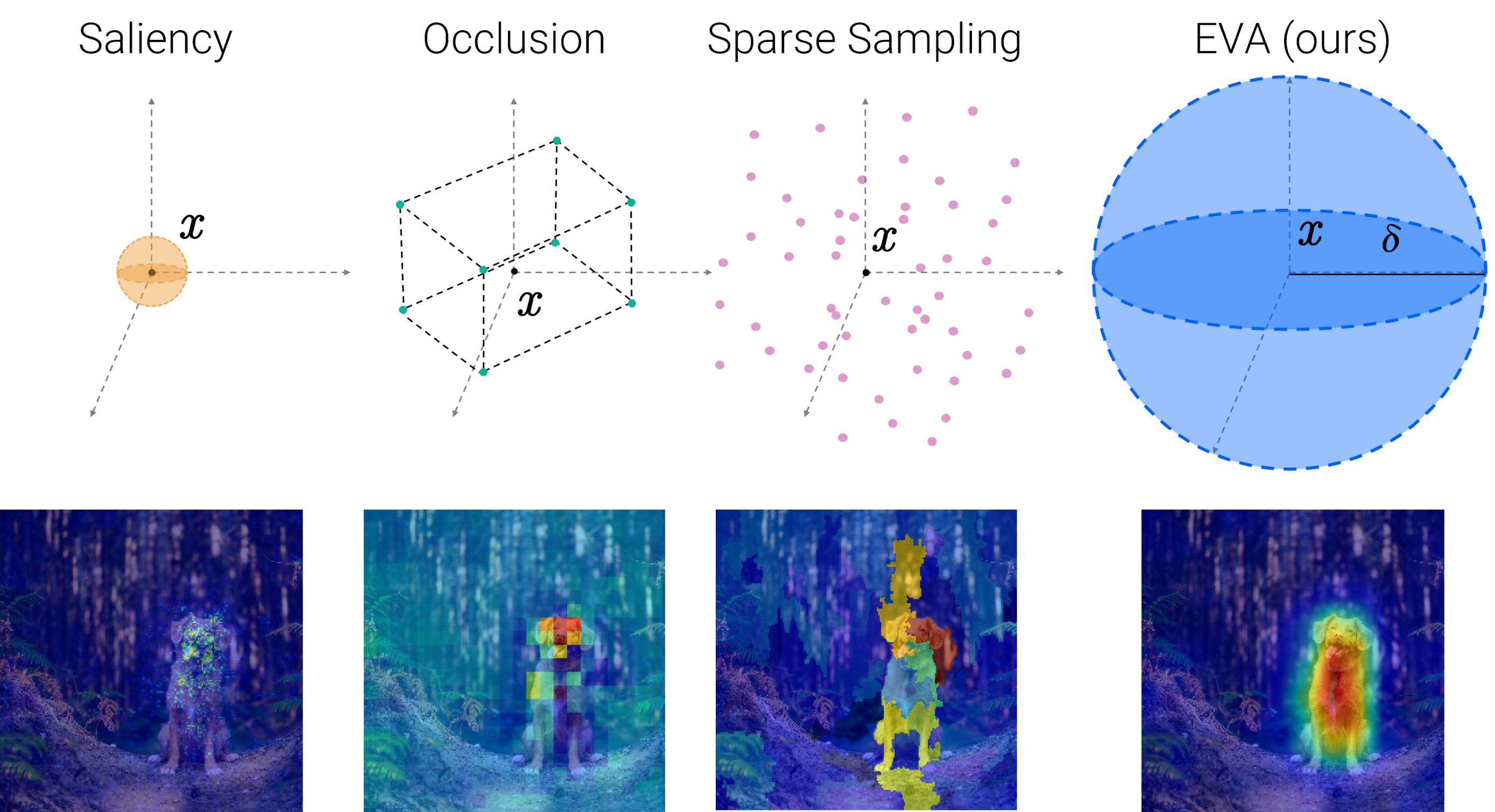}
  \caption{
  \textbf{Manifold exploration of current attribution methods.}
  Current methods assign an importance score to individual pixels using perturbations around a given input image $\vx$. Saliency~\cite{simonyan2013deep} uses infinitesimal perturbations around $\vx$, Occlusion~\cite{zeiler2013visualizing} switches individual pixel intensities on/off. More recent approaches~\cite{ribeiro2016i, lundberg2017unified, petsiuk2018rise, fel2021sobol, novello2022making} use (Quasi-) random sampling methods in specific perturbation spaces (occlusion of segments of pixels, blurring, ...). However, the choice of the perturbation space undoubtedly biases the results -- potentially even introducing serious artifacts~\cite{sturmfels2020visualizing,hsieh2020evaluations,haug2021baselines,kindermans2019reliability}.
  We propose to use verified perturbation analysis to efficiently perform a complete coverage of a perturbation space around $\vx$ to produce reliable and faithful explanations.
  }
  \label{fig:big_picture}
\end{figure}

\vspace{-3mm}
\section{Introduction}
\label{introduction}
%\vspace{-2mm}
\let\thefootnote\relax\footnotetext{\textit{Proceedings of the IEEE / CVF Computer Vision and Pattern Recognition Conference (CVPR), 2023.}}

Deep neural networks are now being widely deployed in many applications from medicine, transportation, and security to finance, with broad societal implications~\cite{lecun2015deep}. They are routinely used to making safety-critical decisions --  often without an explanation as their decisions are notoriously hard to interpret. 

Many explainability methods have been proposed to gain insight into how network models arrive at a particular decision~\cite{zeiler2013visualizing, ribeiro2016i, lundberg2017unified, smilkov2017smoothgrad, shrikumar2017learning, sundararajan2017axiomatic, petsiuk2018rise, Selvaraju_2019, fel2021sobol,novello2022making,graziani2021sharpening}.
The applications of these methods are multiple -- from helping to improve or debug their decisions to helping instill confidence in the reliability of their decisions~\cite{doshivelez2017rigorous}. 
Unfortunately, a severe limitation of these approaches is that they are subject to a confirmation bias: while they appear to offer useful explanations to a human experimenter, they may produce incorrect explanations~\cite{adebayo2018sanity, ghorbani2017interpretation, slack2021counterfactual}.
In other words, just because the explanations make sense to humans does not mean that they actually convey what is actually happening within the model.
Therefore, the community is actively seeking for better benchmarks involving humans~\cite{hsieh2020evaluations,nguyen2021effectiveness,fel2021cannot,kim2021hive}.

In the meantime, it has been shown that some of our current and commonly used benchmarks are biased and that explainability methods reflect these biases -- ultimately providing the wrong explanation for the behavior of the model~\cite{sturmfels2020visualizing,hsieh2020evaluations,hase2021out}.
For example, some of the current fidelity metrics \cite{petsiuk2018rise, aggregating2020,jacovi2020towards,hedstrom2022quantus,fel2021xplique} mask one or a few of the input variables (with a fixed value such as a gray mask) in order to assess how much they contribute to the output of the system. Trivially, if these variables are already set to the mask value in a given image (e.g., gray), masking these variables will not yield any effect on the model's output and the importance of these variables is poised to be underestimated. 
Finally, these methods rely on sampling a space of perturbations that is far too vast to be fully explored -- e.g., LIME on a image divided in $64$ segments image would need more than $10^{19}$ samples to test all possible perturbations. 
As a result, current attribution methods may be subject to bias and are potentially not entirely reliable.

To address the baseline issue, a growing body of work is starting to leverage adversarial methods~\cite{hsieh2020evaluations, boopathy2020proper, lin2019explanations,ross2021learning,idrissi2021developments} to derive explanations that reflect the robustness of the model to local adversarial perturbations. Specifically, a pixel or an image region is considered important 
if it allows the easy generation of an adversarial example. 
That is if a small perturbation of that pixel or image region yields a large change in the model's output. This idea has led to the design of several novel robustness metrics to evaluate the quality of explanations, such as \rsr\cite{hsieh2020evaluations}.
For a better ranking on those robustness metrics, several methods have been proposed that make intensive use of adversarial attacks~\cite{hsieh2020evaluations,yin2022sensitivity}, such as Greedy-AS for \rsr.
However, these methods are computationally very costly -- in some cases, requiring over 50 000 adversarial attacks per explanation -- severely limiting the widespread adoption of these methods in real-world scenarios.

In this work, we propose to address this limitation by introducing \eva~(Explaining using Verified perturbation Analysis), a new explainability method based on robustness analysis. Verified perturbation analysis is a rapidly growing toolkit of methods to derive bounds on the outputs of neural networks in the presence of input perturbations. In contrast to current attributions methods based on gradient estimation or sampling, verified perturbation analysis allows the full exploration of the perturbation space, see Fig.~\ref{fig:big_picture}. We use a tractable certified upper bound of robustness confidence to derive a new estimator to help quantify the importance of input variables (i.e., those that matter the most). That is, the variables most likely to change the predictor's decision.

We can summarize our main contributions as follows:
\begin{itemize}[leftmargin=*]
    \vspace{-2mm}
    \item We introduce \eva, the first explainability method guaranteed to explore its entire set of perturbations using Verified Perturbation Analysis.
    \vspace{-2mm}
    \item We propose a method to scale \eva~to large vision models and show that the exhaustive  exploration of all possible perturbations can be done efficiently.
    \vspace{-2mm}
    \item We systematically evaluate our approach using several image datasets and show that it yields convincing results on a large range of explainability metrics 
    \vspace{-2mm}
    \item Finally, we demonstrate that we can use the proposed method to generate class-specific explanations, and we study the effects of several verified perturbation analysis methods as a hyperparameter of the generated explanations. 

\end{itemize}

\vspace{-3mm}

\section{Related Work}
\label{related_work}

\paragraph{Attribution Methods.}

Our approach builds on prior attribution methods in order to explain the prediction of a deep neural network via the identification of input variables that support the prediction (typically pixels or image regions for images -- which lead to importance maps shown in Fig.~\ref{fig:big_picture}). ``Saliency'' was first introduced in~\cite{baehrens2010explain} and consists in using the gradient of a classification score. It was later refined in~\cite{simonyan2014deep, zeiler2014visualizing, springenberg2014striving, sundararajan2017axiomatic, smilkov2017smoothgrad} in the context of deep convolutional networks for classification. 
However, the image gradient only reflects the model's operation within an infinitesimal neighborhood around an input.Hence, it can yield misleading importance estimates~\cite{ghalebikesabi2021locality} since gradients of the current large vision models are noisy~\cite{smilkov2017smoothgrad}. 
Other methods rely on different image perturbations applied to images to  produce  importance maps that reflect the corresponding change in classification score resulting from the perturbation.  
Methods such as ``Occlusion''~\cite{zeiler2014visualizing}, ``LIME''~\cite{ribeiro2016i}, ``RISE''~\cite{petsiuk2018rise}, ``Sobol''~\cite{fel2021sobol} or ``HSIC''~\cite{novello2022making} leverage different sampling strategies to explore the space of perturbations around the image.
For instance, Occlusion uses binary masks to occlude individual image regions, one at a time. RISE and HSIC combines these discrete masks to perturb multiple regions simultaneously. Sobol uses continuous masks for a finer exploration of the perturbation space.

Nevertheless, none of these methods are able to systematically cover the full space of perturbations.
As a result, the corresponding explanations may not reliably reflect the true importance of pixels.
In contrast, our approach comes with strong guarantees that can be derived from the verified perturbation analysis framework as it provides bounds on the perturbation space.

\vspace{-2mm}
\paragraph{Robustness-based Explanation.} To try to address the aforementioned limitations, several groups~\cite{ignatiev2019abduction, ignatiev2019relating, slack2021reliable, hsieh2020evaluations, boopathy2020proper, lin2019explanations, fel2020representativity} have focused on the development of a new set of robustness-based evaluation metrics for trustworthy explanations. 
These new metrics are in contrast with the previous ones, which consisted in removing the pixels considered important in an explanation by substituting them with a fixed baseline -- which inevitably introduces bias and artifacts~\cite{hsieh2020evaluations,sturmfels2020visualizing,haug2021baselines,kindermans2019reliability,hase2021out}. 
Key to these new metrics is the assumption that when the important pixels are in their nominal (fixed) state, then perturbations applied to the complementary pixels -- deemed unimportant -- should not affect the model's decision to any great extent. The corollary that follows is that perturbations limited to the pixels considered important should easily influence the model's decision~\cite{lin2019explanations,hsieh2020evaluations}.
Going further along the path of robustness, abductive reasoning was used in~\cite{ignatiev2019abduction} to compute optimal subsets with guarantees.  The challenge consists  in looking for the subset with the smallest possible  cardinality -- to guarantee the decision of the model. This work constituted one of the early successes of formal methods for explainability, but the approach was limited to low-dimensional problems and shallow neural networks. It was later extended to relax the subset minimum explanation by either providing multiple explanations, aggregating pixels in bundles~\cite{bassan2022towards} or by using local surrogates~\cite{boumazouza2021asteryx}.

Some heuristics-oriented works also propose to optimize these new robustness based criteria and design new methods using a generative model~\cite{o2020generative} or adversarial attacks~\cite{hsieh2020evaluations}.
The latter approach requires searching for the existence or lack of an adversarial example for a multitude of $\ell_p$ balls around the input of interest. As a result, the induced computational cost is quite high as the authors used more than $50000$  computations of adversarial examples to generate a single explanation. 

More importantly, a failure to find an adversarial perturbation for a given radius does not guarantee that none exists. In fact, it is not uncommon for adversarial attacks to fail to converge --  or fail to find an adversarial example -- which will result in a failure to output an importance score.
Our method addresses these issues while drastically reducing the computation cost.
An added benefit of our approach is that verified perturbation analysis provides additional guarantees and hence opens the doors of certification which is a necessity for safety-critical applications.

\vspace{-3mm}

\paragraph{Verified Perturbation Analysis.} This growing field of research focuses on the development of methods that outer-approximate neural network outputs given some input perturbations. 
Simply put, for a given input $\vx$ and a bounded perturbation $\vd$, verification methods yield minimum $\underline{\pred}(\vx)$ and maximum $\overline{\pred}(\vx)$ bounds on the output of a model. Formally $\forall~ \vd ~s.t~ ||\vd||_p \leq \varepsilon$:

\vspace{-2mm}
$$
\underline{\pred}(\vx) \leq \pred(\vx + \vd) \leq \overline{\pred}(\vx). 
$$
This allows us to explore the whole perturbation space without having to explicitly sample points in that space.

Early works focused on computing reachable lower and upper bounds based on satisfiability modulo theories~\cite{katz2017reluplex, ehlers2017formal}, and mixed-integer linear programming problems~\cite{tjeng2017verifying}. While these early results were encouraging, the proposed methods struggled even for small networks and image datasets. More recent work has led to the independent development of methods for computing looser certified lower and upper bounds more efficiently thanks to convex linear relaxations either in the primal or dual space~\cite{salman2019convex}. 
While looser, those bounds remain tight enough to yield non-ubiquitous robustness properties on medium size neural networks. CROWN (hereafter called Backward) uses Linear Relaxation-based Perturbation Analysis (LiRPA) and achieves the tightest bound for efficient single neuron linear relaxation~\cite{zhang2018efficient, singh2019abstract, wang2021beta}. 
In addition, linear relaxation methods offer a wide range of possibilities with a vast trade-off between ``tigthness'' of the bounds and efficiency. 
These methods form two broad classes: `forward' methods which propagate constant bounds (more generally affine relaxations from the input to the output of the network) also called Interval Bound Propagation (IBP, Forward, IBP+Forward) vs. `backward' methods which bound the output of the network by affine relaxations given the internal layers of the network, starting from the output to the input. Note that these methods can be combined, e.g. (CROWN + IBP + Forward).
For a thorough description of the LiRPA framework and theoretical analysis of the worst-case complexities of each variant, see~\cite{xu2020automatic}.
In this work, we remain purposefully agnostic to the verification method used and opt for the most accurate LiRPA method applicable to the predictor. Our approach is based on the formal verification framework DecoMon, based on Keras~\cite{ducoffe2021decomon}.

\section{Explainability with Verified Perturbation Analysis}
\label{framework}

\vspace{-2mm}
\paragraph{Notation.} We consider a standard supervised machine-learning classification setting with input space $\mathcal{X} \subseteq \mathbb{R}^d$, an output space $\mathcal{Y} \subseteq \mathbb{R}^c$, and a predictor function $\pred : \mathcal{X} \to \mathcal{Y}$ that maps an input vector $\vx~=~(x_1,~\ldots{},~x_d)$ to an output $\pred(\vx)~=~\left(\pred_1(\vx),~\ldots{},~\pred_c(\vx)\right)$.
We denote $\ball~=~\{\vd \in \mathbb{R}^d ~:~  || \vd||_{p} \leq \varepsilon \}$ 
the perturbation ball with radius $\varepsilon > 0$,  with 
$p \in \{1, 2, \infty\}$.
For any subset of indices $\vu\subseteq\{1,~\ldots{},~d\}$, we denote 
$\ballu$ the ball without perturbation on the variables in $\vu$:
$\ballu = \{ \vd ~:~ \vd \in \ball, ~ \vd_{\vu} = 0 \}$ and $\ball(\bm{x})$ the perturbation ball centered on $\bm{x}$. We denote the lower (resp. upper) bounds obtained with verification perturbation analysis as $\underline{\pred}(\vx, \ball)~=~\left(\underline{\pred}_1(\vx, \ball),~\ldots{},~\underline{\pred}_c(\vx, \ball)\right)$, and  
$\overline{\pred}(\vx, \ball) = \left(\overline{\pred}_1(\vx, \ball),~\ldots{},~\overline{\pred}_c(\vx, \ball)\right)$. 
Intuitively, these bounds delimit the output prediction for any perturbed sample in $\ball(\vx)$.

\subsection{The importance of setting the importance right}

Different attribution methods implicitly assume different definitions of the notion of importance for input variables based either on game theory~\cite{lundberg2017unified}, the notion of conditional expectation of the score logits~\cite{petsiuk2018rise}, their variance~\cite{fel2021sobol} or on some  measure of statistical dependency between different areas of an input image and the output of the model~\cite{novello2022making}.
In this work, we build on robustness-based explainability methods~\cite{hsieh2020evaluations} which assume that a variable is important if small perturbations of this variable lead to large changes in the model decision.
Conversely, a variable is said to be unimportant if changes to this variable only yield small changes in the model decision.
From this intuitive assertion, we construct an estimator that we call \Adv.

\subsection{Adversarial overlap}

\vspace{-1mm}
We go one step beyond previous work and propose to compute importance by taking into account not only the ability of individual variables to change the network's decision but also its confidence in the prediction.
\Adv ~ measures the extent to which a modification on a group of pixels can generate overlap between classes, i.e. generate a point close to $\vx$ such that the attainable maximum of an unfavorable class $c'$ can match the minimum of the initially predicted class $c$.

Indeed, if a modification of a pixel -- or group of pixels -- allows generating a new image that changes the decision of $\pred$, this variable must be considered important. 
Conversely, if the decision does not change regardless of the value of the pixel, then the pixel can be left at its nominal value and should be considered unimportant. 

Among the set of possible variable perturbations $\vd$ around a point $\vx$, we, therefore, look for points that can modify the decision with the most confidence.
Hence our scoring criterion can be formulated as follows:
\vspace{-2mm}
\begin{equation}\label{eq:adv_surface}
    \AO_c(\vx, \mathcal{B}) = 
    \max_{\substack{\vd \in \mathcal{B}\\c'\neq{}c}} \pred_{c'}(\vx + \vd) - \pred_c(\vx + \vd).
\end{equation}
\vspace{-3mm}

Intuitively, this score represents the confidence of the ``best'' adversarial perturbation that can be found in the perturbation ball $\ball$ around $\vx$.
Throughout the article, when $c$ is not specified, it is assumed that $c = \argmax \pred(\vx)$.

In order to estimate this criterion, a naive strategy could be to use adversarial attacks to search within $\ball$. 
However, when they converge - which is not ensured, such methods only explore certain points of the considered space, thus giving no guarantee regarding the optimality of the solution. 
Moreover, adversarial methods have no guarantee of success and therefore cannot ensure a valid score under every circumstance.
Finally, the large dimensions of the current datasets make exhaustive searches impossible.

To overcome these issues, we take advantage of one of the main results from verified perturbation analysis to derive a guaranteed upper bound on the criterion introduced in Eq.~\ref{eq:adv_surface}. 
We can upper bound the \adv{} criterion as follows: 
$$ 
\AO(\vx, \ball) \leq \AOup(\vx, \ball) = \max\limits_{c'\neq c} \overline{\pred}_{c'}(\vx, \ball) - \underline{\pred}_c(\vx, \ball). 
$$
The computation of this upper bound becomes tractable using any verified perturbation analysis method.

For example, $\AOup(\vx, \ball) \leq 0$  guarantees that no adversarial perturbation is possible in the perturbation space.\footnote{Note that with adversarial attacks, failure to find an adversarial example does not guarantee that it does not exist.}
Our upper bound $\AOup(\vx, \ball)$ corresponds to the difference between the verified lower bound      of the class of interest $c$ and the maximum over the verified upper bounds among the other classes.
Thus, when important variables are modified (e.g the head of the dog in Fig.~\ref{fig:eva}, using $\textcolor{pink}{\ball}$), the lower bound for the class of interest will get smaller than the upper bound of the adversary class. On the other hand, this overlap is not possible when important variables are fixed (e.g in Fig.~\ref{fig:eva} when the head of the dog is fixed, using $\textcolor{indigo}{\ballu}$).
We now demonstrate how to leverage this score to derive an efficient estimator of variable importance.

\begin{figure*}[t!]
  \includegraphics[width=0.99\textwidth]{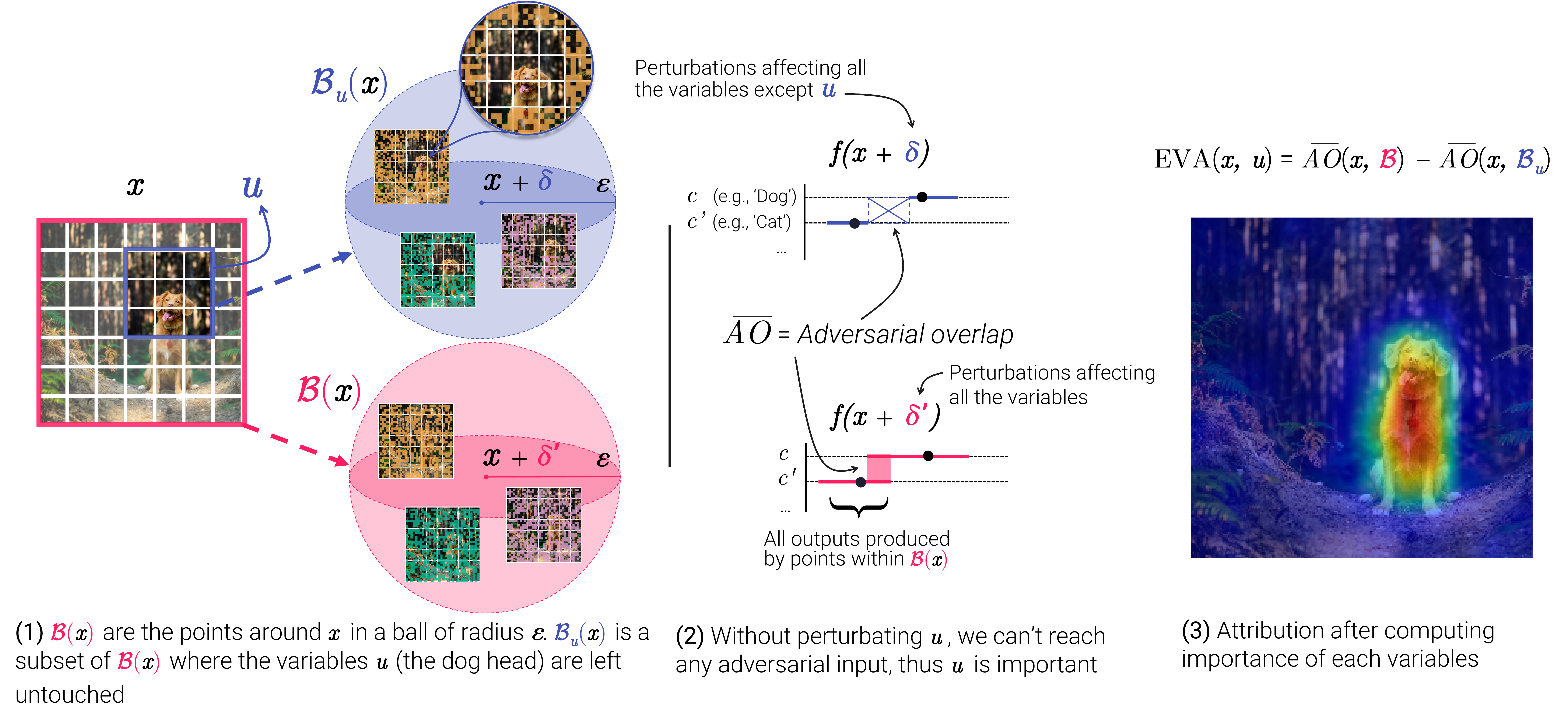}
  \caption{
  \textbf{\eva~attribution method.} In order to compute the importance for a group of variables $\vu$ -- for instance the dog's head -- the first step (1) consists in designing the perturbation ball $\textcolor{indigo}{\ballu}(\vx)$. This ball is centered in $\vx$ and contain all the possible images perturbed by $\textcolor{indigo}{\vd} ~s.t~ ||\textcolor{indigo}{\vd}||_{p} \leq \varepsilon, ||\textcolor{indigo}{\vd}_{\vu}||_p = 0$ which do not perturb the variables $\textcolor{indigo}{\vu}$. Using verified perturbation analysis, we then compute the \adv~ $\AOup(\vx, \textcolor{indigo}{\ballu})$ which corresponds to the overlapping between the class $c$ -- here dog -- and $c'$, the maximum among the other classes. Finally, the importance score for the variable $\vu$ corresponds to the drop in \adv~ when $\vu$ cannot be perturbed, thus the difference between $\AOup(\vx, \textcolor{pink}{\ball})$ and $\AOup(\vx, \textcolor{indigo}{\ballu})$. 
  Specifically, this measures how important the variables $\vu$ are for changing the model's decision.
  }
  \vspace{-3mm}
  \label{fig:eva}
\end{figure*}
\subsection{\eva}

We are willing to assign a higher importance score for a variable allowing (1) a change in a decision, (2) a greater adversarial -- thus a solid change of decision. Modifying all variables gives us an idea of the robustness of the model.
In the same way, the modification of all variables without the subset $\vu$ allows quantifying the change of the strongest adversarial perturbation and thus quantifies the importance of the variables $\vu$. Intuitively, if an important variable $\vu$ is discarded, then it will be more difficult, if not impossible, to succeed in finding any adversarial perturbation. Specifically, removing the possibility to modify $\vx_{\vu}$ allows us to reveal its importance by taking into account its possible interactions.

The complexity of current models means that the variables are not only treated individually in neural network models but collectively. In order to capture these higher-order interactions, our method consists in measuring the \adv~ allowed by all the variables together $\AOup(\vx, \ball)$ -- thus taking into account their interactions -- and then forbidding to play on a group of variables $\AOup(\vx, \ballu)$ to estimate the importance of $\vu$. Making the interactions of $\vu$ disappear reveals their importance. Note that several works have mentioned the importance of taking into account the interactions of the variables when calculating the importance~\cite{petsiuk2018rise,fel2021sobol, ferrettini2021coalitional,idrissi2023coalitional}. Formally, we introduce \eva~(Explainability using Verified perturbation Analysis) that measure the drop in \adv~ when we fixed the variables $\vu$:
\vspace{-1mm}
\begin{equation}
    \label{eq:tod_estimator}
    \bm{\eva}(\vx, \vu, \ball) \equiv \AOup(\vx, \ball) - \AOup(\vx, \ballu).
\end{equation}
\vspace{-5mm}

As explained in Fig.~\ref{fig:eva}, the estimator requires two passes of the perturbation analysis method; one for $\AOup(\ball)$, and the other for $\AOup(\ballu)$: the first term consists in measuring the \adv~ by modifying all the variables, the second term measures the adversarial surface when fixing the variables of interest $\vu$.
In other words, \eva~measures the \adv~that would be left if the variables $\vu$ were to be fixed.

From a theoretical point of view, we notice that \eva~- under reasonable assumptions - yield the optimal subset of variables to minimize the theoretical \rsr~ metric (see Theorem~\ref{thm:rsr}).
From a computational point of view, we can note that the first term of the \adv~$\AOup(\vx, \ball)$ -- as it does not depend on $\vu$ -- can be calculated once and re-used to evaluate the importance of any other variables considered. 
Moreover, contrary to an iterative process method~\cite{Fong_2017, hsieh2020evaluations, ignatiev2019abduction}, each importance can be evaluated independently and thus benefit from the parallelization of modern neural networks. Finally, the experiments in Section~\ref{sec:experiments} show that even with two calls to $\AOup$~ per variables, our method remains much faster than the one based on sampling or on adversarial attacks (such as Greedy-AS or Greedy-AO, see appendix \ref{ap:benchmarks}). %

In this work, the verified perturbation-based analysis considered is not always adapted to high dimensional models, especially those running on ImageNet~\cite{imagenet_cvpr09}. We are confident that the verification methods will progress towards more scalability in the near future, enabling the original version of \eva~ on deeper models. 

In the meantime, we introduce an empirical method that allows to scale \eva\xspace to high dimensional models. This method sacrifices theoretical guarantees, but the results section reveals that it may be a good compromise.

\begin{figure*}[t!]
  \includegraphics[width=0.9\textwidth]{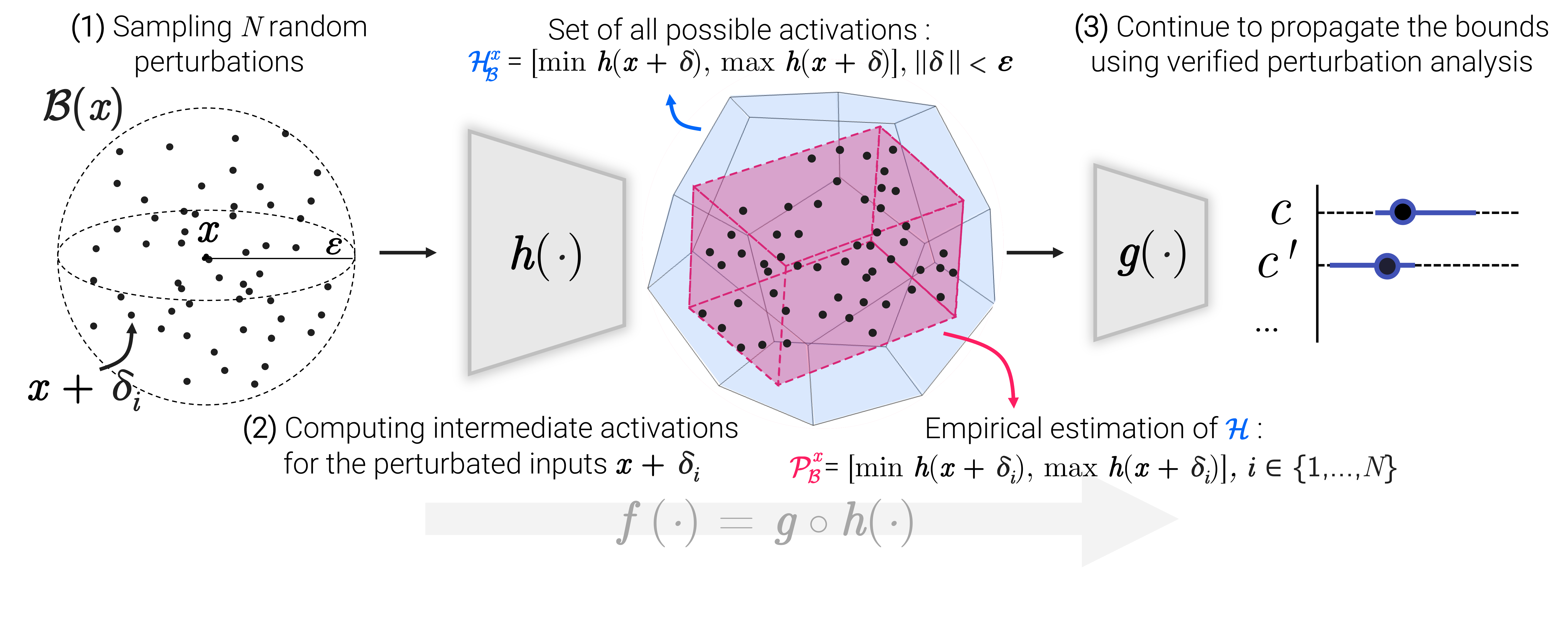}\vspace{-7mm}
  \caption{\textbf{Scaling strategy.} 
  In order to scale to very large models, we propose to estimate the bounds of an intermediate layer's activations empirically by (1) Sampling $N$ input perturbations and (2) calculating empirical bounds on the resulting activations for the layer $\bm{h}(\cdot)$. 
  We can then form the set $\textcolor{pink}{\mathcal{P}_{\ball}^{\vx}}$ which is a subset of the true bounds $\textcolor{indigo}{\mathcal{H}_{\ball}^{\vx}}$ since the sampling is never exhaustive. We can then plug this set into a verified perturbation analysis method (3) and continue the forward propagation of the inputs through the rest of the network.
  }
  \label{fig:eva_hybrid}
\end{figure*}

\subsection{Scaling to larger models}
\label{sec:scaling}

\vspace{-1mm}

We propose a second version of \eva, 
which is a combination of sampling and verification perturbation analysis. 
The aim of this hybrid method is twofold: (\textit{\textbf{i}}) take advantage of sampling to approach the bounds of an intermediate layer in a potentially very large model, (\textit{\textbf{ii}}) then complete only the rest of the propagations with verified perturbation analysis and thus move towards the native \eva\xspace method which benefits from theoretical guarantees.
Note that, combining verification methods with empirical methods (a.k.a adversarial training) has notably been proposed in~\cite{balunovic2019adversarial} for robust training.

Specifically, our technique consists of splitting the model into two parts, and (\textit{\textbf{i}}) estimating the bounds of an intermediate layer using sampling, (\textit{\textbf{ii}}) propagating these empirical intermediate bounds onto the second part of the model with verified perturbation analysis methods.

For the first step (\textit{\textbf{i}}) we consider the original predictor $\pred$ as a composition of functions 
$\pred(\vx) = \bm{g} \circ \bm{h}(\vx)$. 
For deep neural networks, $\bm{h}(\cdot)$ is a function that maps input to an intermediate feature space 
and $\bm{g}(\cdot)$ is a function that maps this same feature space to the classification.

\vspace{-1mm}
We propose to empirically estimate bounds 
$( \underline{\bm{h}}_{\ball}^{\vx}, \overline{\bm{h}}_{\ball}^{\vx} )$ for the intermediate activations $ \bm{h}(\cdot) \in \mathbb{R}^{d'}$ using Monte-Carlo sampling on the perturbation $\vd \in \ball$. Formally:
\vspace{-1mm}
\begin{equation*}
\begin{split}
    \forall j \in [0, \ldots, d'], ~
& \underline{\bm{h}}_{\ball}^{\vx}[j] = 
  \min\limits_{\vd_1,\ldots \vd_i,\ldots \vd_n \overset{\mathrm{iid}}{\sim} U(\mathcal{B})}
  \bm{h}(\vx+\vd_i)[j]\\
& \overline{\bm{h}}_{\ball}^{\vx}[j] =  
  \max\limits_{\vd_1,\ldots \vd_i,\ldots \vd_n  \overset{\mathrm{iid}}{\sim} U(\mathcal{B})} 
  \bm{h}(\vx+\vd_i)[j].
\end{split}
\end{equation*}

Obviously, since the sampling is never exhaustive, the bounds obtained underestimate the true maximum $ \overline{\bm{h}}_{\ball}^{\vx} \leq \max \bm{h}(\vx + \vd)$ and overestimates the true minimum $ \underline{\bm{h}}_{\ball}^{\vx} \geq \min \bm{h}(\vx + \vd)$ as illustrated in the Fig.~\ref{fig:eva_hybrid}.
In a similar way, we define $\underline{\bm{h}}_{\ballu}^{\vx}$ and $\overline{\bm{h}}_{\ballu}^{\vx}$ when  $\vd \in \ballu$. 
Once the empirical bounds are estimated, we may proceed to the second step and use the obtained bounds to form the new perturbation set 
\vspace{-2mm}
$$ 
\mathcal{P}^{\vx}_{\ball} = 
    [\underline{\bm{h}}_{\ball}^{\vx} - \bm{h}(\vx), 
    \overline{\bm{h}}_{\ball}^{\vx} - \bm{h}(\vx)].
$$ 
Intuitively, this set bounds the intermediate activations obtained empirically and can then be fed to a verified perturbation verification method.

We then carry out the end of the bounds propagation in the usual way, using verified perturbation analysis. This amounts to computing bounds for the outputs of the network for all possible activations contained in our empirical bounds. The only change is that we no longer operate in the pixel space $\vx$ with the ball $\ball$, but in the activation space $\bm{h}(\cdot)$  with the perturbations set $\mathcal{P}^{\vx}_{\ball}$. The importance score of a set of variables $\vu$ is then : 
\vspace{-2mm}
$$ \evaH(\vx, \vu, \ball) \equiv \eva(\bm{h}(\vx), \vu, \mathcal{P}^{\vx}_{\ball}).$$

\vspace{-1mm}
This hybrid approach allows us to use \eva~ on state-of-the-art models and thus to benefit from our method while remaining tractable. We believe this extension to be a promising step towards robust explanations on deeper networks.

\section{Experiments}
\label{sec:experiments}

\begin{table*}[t]
    \centering
    \scalebox{0.87}{
        \begin{tabular}{l C{0mm} P{0mm} P{5mm}P{5mm}P{5mm}P{5mm}P{5mm} P{1mm} P{5mm}P{5mm}P{9mm}P{5mm}P{5mm} P{1mm} P{5mm}P{5mm}P{5mm}P{5mm}P{5mm}}
        \toprule
        &&& \multicolumn{5}{c}{MNIST} &&  \multicolumn{5}{c}{Cifar-10} &&  \multicolumn{5}{c}{ImageNet}  \\
        \cmidrule(lr){4-8} \cmidrule(lr){10-14} \cmidrule(lr){16-20}
        
        &&& Del.$\downarrow$ & Ins.$\uparrow$ & Fid.$\uparrow$ & Rob.$\downarrow$ & Time 
        && Del.$\downarrow$ & Ins.$\uparrow$ & Fid.$\uparrow$ & Rob.$\downarrow$ & Time 
        && Del.$\downarrow$ & Ins.$\uparrow$ & Fid.$\uparrow$ & Rob.$\downarrow$ & Time
        \\
        
        \midrule
        Saliency\cite{simonyan2013deep} &   && .193 & \underline{.633} & \underline{.378} & .071 & 0.04
                 && \underline{.171} & .172 & -.021 & .026 & 0.16
                 && \underline{.057} & .126 & .035 & .769 & 0.36 \\
        GradInput\cite{ancona2017better} &   && .222 & .611 & .107 & .074 & 0.04 
                 && .200 & .143 & -.018 & .095 & 0.17
                 && \underline{.057} & .050 & .023 & .814 & 0.36  \\
        SmoothGrad\cite{smilkov2017smoothgrad} &   && .185 & .621 & .331  & .070 & 1.91
                 && .174 & .181 & .092 & .048 & 9.07
                 && \textbf{.051} & .069 & .019 & .809 & 9.63 \\
        VarGrad\cite{seo2018noise} &   && .207   & .555   & .216   & .077   & 1.76 
                 && .183 & .211 & -.012 & .193 & 9.07
                 && .098 & .201 & .021 & .787 &  9.62 \\
        InteGrad\cite{sundararajan2017axiomatic} &   && .209 & .615 & .108 & .074 & 1.77
                 && .194 & .171 & -.016 & .154 & 7.19
                 && .058 & .052 & .023 & .813 & 8.39 \\
        Occlusion\cite{ancona2017better} &   && .247  & .545  & .137  & .082  & 0.04 
                 && .217 & \underline{.290} & .105 & .232 & 1.13
                 && .100 & .266 & .026 & .821 & 4.97 \\
        GradCAM\cite{Selvaraju_2019} &   && n/a &  n/a  &  n/a  &  n/a  &  n/a  
                 && .297 & .282 & .056 & .195 & 0.39
                 && .073 & .232 & .036 & .817 & 0.18 \\
        GradCAM++\cite{chattopadhay2018grad} &   && n/a  & n/a & n/a & n/a & n/a
                 && .270 & \textbf{.326} & .102 & .094 & 0.39
                 && .074 & \underline{.285} & \underline{.054} & .800 & 0.19 \\
        RISE\cite{petsiuk2018rise}     &   && .248  & .558 & .133  & .093  & 2.26 
                 && .196 & .273  & \underline{.157} & .385 & 20.5
                 && .074 & .276 & \textbf{.154} & .818 & 1215 \\
        Greedy-AS\cite{hsieh2020evaluations} &   && .260  & .497  & .110  & \textbf{.061}  & 335
                 && .205 & .264 & -.003 & \textbf{.013} & 4618
                 && .088 & .047 & .023 & \textbf{.612} & 180056  \\
        \midrule 
        \textbf{\eva}~(ours)     & && \textbf{.089} & \textbf{.736} & \textbf{.428} & \underline{.069} & 1.29 
                 && \textbf{.164} & \underline{.290} & \textbf{.352} & \underline{.025} & 12.7 
                 && .070 & \textbf{.289} & .048 & \underline{.758} & 6454
                 \\
        
        \bottomrule \\
        \end{tabular}
    }
    \caption{Results on Deletion (Del.), Insertion (Ins.), $\mu$Fidelity (Fid.) and \rsr~ (Rob.) metrics. 
Time in seconds corresponds to the generation of 500 (MNIST/CIFAR-10) and 100 (ImageNet) explanations on an Nvidia P100.
Note that \eva~is the only method with guarantees that the entire set of possible perturbations has been exhaustively searched.
Verified perturbation analysis with IBP + Forward + Backward is used for MNIST, with Forward only for  CIFAR-10 and with our hybrid strategy described in Section.\ref{sec:scaling} for ImageNet. 
Grad-CAM and Grad-CAM++ are not calculated on the MNIST dataset since the network  only has dense layers. 
The first and second best results are  in \textbf{bold} and \underline{underlined}, respectively. 
}
    \label{tab:cifar_mnist_metrics}
    \vspace{-3mm}
\end{table*}

To evaluate the benefits and reliability of our explainability method, we performed several experiments on a standard dataset, using a set of common explainability metrics against \eva.
In order to test the fidelity of the explanations produced by our method, we compare them to that of 10 other explainability methods using the (1) Deletion, (2) Insertion, and (3) MuFidelity metrics. As it has been shown that these metrics can exhibit biases, we completed the benchmark by adding the (4) \rsr\xspace metric. Each score is averaged over 500 samples.

We evaluated these 4 metrics on 3 image classification datasets, namely MNIST~\cite{lecun2010mnist}, CIFAR-10~\cite{krizhevsky2009learning} and ImageNet~\cite{imagenet_cvpr09}.
Through these experiments, the explanations were generated using \eva~estimator introduced in Equation~\ref{eq:tod_estimator}. The importance scores were not evaluated pixel-wise but on each cell of the image after having cut it into a grid of 12 sides (see Fig.~\ref{fig:eva}). For MNIST and Cifar-10, we used 
$\varepsilon = 0.5$, whereas for ImageNet $\varepsilon = 5$. Concerning the verified perturbation analysis method, we used (IBP+Forward+Backward) for MNIST, and (IBP+Forward) on Cifar-10 and $p=\infty$. For computational purposes, we used the hybrid approach introduce in Section~\ref{sec:scaling} for ImageNet using the penultimate layer (FC-4096) as the intermediate layer $\bm{h}(\cdot)$. We give in Appendix the complete set of hyperparameters used for the other explainability methods, metrics considered as well as the architecture of the models used on MNIST and Cifar-10.

\subsection{Comparison with the state of the art}

There is a general consensus that fidelity is a crucial criterion for an explanation method. That is, if an explanation is used to make a critical decision, then users are expecting it to reflect the true decision-making process underlying the model and not just a consensus with humans. Failure to do so could have disastrous consequences. Pragmatically, these metrics assume that the more faithful an explanation is, the faster the prediction score should drop when pixels considered important are changed.
In Table~\ref{tab:cifar_mnist_metrics}, we present the results of the Deletion~\cite{petsiuk2018rise} (or $1 - AOPC$~\cite{samek2016evaluating}) metric for the MNIST and Cifar-10 datasets on 500 images sampled from the test set. TensorFlow~\cite{tensorflow2015} and the Keras API~\cite{chollet2015keras} were used to run the models and Xplique~\cite{fel2021xplique} for the explainability methods.
In order to evaluate the methods, the metrics require a baseline and several were proposed~\cite{sturmfels2020visualizing, hsieh2020evaluations}, but we chose to keep the choice of~\cite{hsieh2020evaluations} using their random baseline.
 
We observe that \eva~is the explainability method getting the best Deletion, Insertion, and $\mu$Fidelity scores on MNIST, and is just behind Greedy-AS on \rsr. This can be explained by the fact that the Robustness metric uses the adversarial attack PGD~\cite{madry2017pgd}, which is the same one used to generate Greedy-AS, thus biasing the adversarial search. Indeed, if PGD does not find an adversarial perturbation using a subset $\vu$ does not give a guarantee of the robustness of the model, just that the adversarial perturbation could be difficult to reach with PGD.

For Cifar-10, \eva~remains overall the most faithful method according to Deletion and $\mu$Fidelity, and obtains the second score in Insertion behind Grad-Cam++~\cite{chattopadhay2018grad}. 
Finally, we notice that if Greedy-AS~\cite{hsieh2020evaluations} allows us to obtain a good \rsr\xspace score, but this comes with a considerable computation time, which is not the case of \eva\xspace which is much more efficient. Eventually, EVA is a very good compromise for its relevance to commonly accepted explainability metrics and more recent robustness metrics.

\vspace{-4mm}
\paragraph{ImageNet}

After having demonstrated the potential of the method on vision datasets of limited size, we consider the case of ImageNet which has a significantly higher level of dimension.
The use of verified perturbation analysis methods other than IBP is not easily scalable on these datasets. We, therefore, used the hybrid method introduced in Section ~\ref{sec:scaling} in order to estimate the bounds in a latent space  and then plug those bounds into the perturbation analysis to get the final \adv~ score.

Table~\ref{tab:cifar_mnist_metrics} shows the results obtained with the empirical method proposed in Section~\ref{sec:scaling}. We observe that even with this relaxed estimation, \eva~is able to score high on all the metrics. Indeed, \eva~obtains the best score on the Insertion metric and ranks second on $\mu$Fidelity and 
\rsr.
Greedy-AS ranks first on \rsr\xspace at the expense of the other scores where it performs poorly. Finally, both RISE and SmoothGrad perform  well on all the fidelity metrics but collapse on the robustness metric. Extending results with ablations of \eva, including Greedy-AO, are available in Table~\ref{tab:ablation_ao}.

Qualitatively, Fig.~\ref{fig:imagenet_explanations} shows examples of explanations produced on the ImageNet VGG-16 model. The explanations produced by \eva~ are more localized than Grad-CAM or RISE, while being less noisy than the gradient-based or Greedy-AS methods.

In addition, as the literature on verified perturbation analysis is evolving rapidly we can conjecture that the advances will benefit the proposed explainability method.
Indeed, \eva~proved to be the most effective on the benchmark when an accurate formal method was used.
After demonstrating the performance of the proposed method, we study its ability to generate class explanations specific. %

\subsection{Tighter bounds lead to improved explanations}
\label{sec:vpa_ablation}

\vspace{-1mm}
\begin{table}[h]
    \centering
    \scalebox{0.9}{
        \begin{tabular}{l c P{5mm}P{5mm}P{5mm}P{5mm}}
        \toprule
        & Tightness$\downarrow$ & Del.$\downarrow$ & Ins.$\uparrow$ & Fid.$\uparrow$ & Rob.$\downarrow$ \\
        \midrule
        IBP & 4.58 & .148 & .588 & .222 & .077 \\
        Forward & 2.66 & .150 & .580 & .209 & .078 \\
        Backward & \underline{2.36} &  \underline{.115} & \underline{.607} & \underline{.274} & \underline{.074} \\
        IBP + Fo. + Ba. & \textbf{1.55} & \textbf{.089} & \textbf{.736} & \textbf{.428} & \textbf{.069} \\
        \bottomrule
        \vspace{-4mm}
        \end{tabular}
    }
    \caption{\textbf{Impact of the verified perturbation analysis method on EVA.}
    Results of \eva on Tightness, Deletion (Del.), Insertion (Ins.), Fidelity (Fid.) and $\rsr$ (Rob.) metrics obtained on MNIST. The Tightness score corresponds to the average adversarial surface. A lower Tightness score indicates that the method is more precise: it reaches tighter bound, resulting in better explanations and superior scores on the other metrics. The first and second best results are respectively in \textbf{bold} and \underline{underlined}.
    \vspace{-4mm}
    }
    \label{tab:ablation_vpa}
\end{table}

The choice of the verified perturbation analysis method is a hyperparameter of EVA.
Hence, it is interesting to see the effect of the choice of this hyperparameter on the previous benchmark. 
We recall that only the MNIST dataset could benefit from the (IBP+Forward+Backward) combo. Table~\ref{tab:ablation_vpa}  reports the results of the fidelity metrics using other verified perturbation analysis methods. We also report a tightness score which corresponds to the average of the \adv~: 
$\mathbb{E}_{\vx \sim \mathcal{X}}(\AOup(\vx, \ball))$.
Specifically, a low score indicates that the verification method is precise, meaning that the over-approximation is closer to the actual value. It should be noted that the true value is intractable, but remains the same across all three tested cases. We observe that the tighter the bounds, the higher the scores. This allows us to conjecture that the more scalable the formal methods will become, the better the quality of the generated explanations will be.
We perform additional experiments to ensure that the certified component of \eva\xspace score is significant by comparing \eva\xspace to 
a sampling-based version of \eva. The details of these experiments are available in Appendix~\ref{ap:benchmarks}.

\subsection{Targeted Explanations}
\label{sec:targeted_explanations}

\begin{figure}[t!]
  \includegraphics[width=0.45\textwidth]{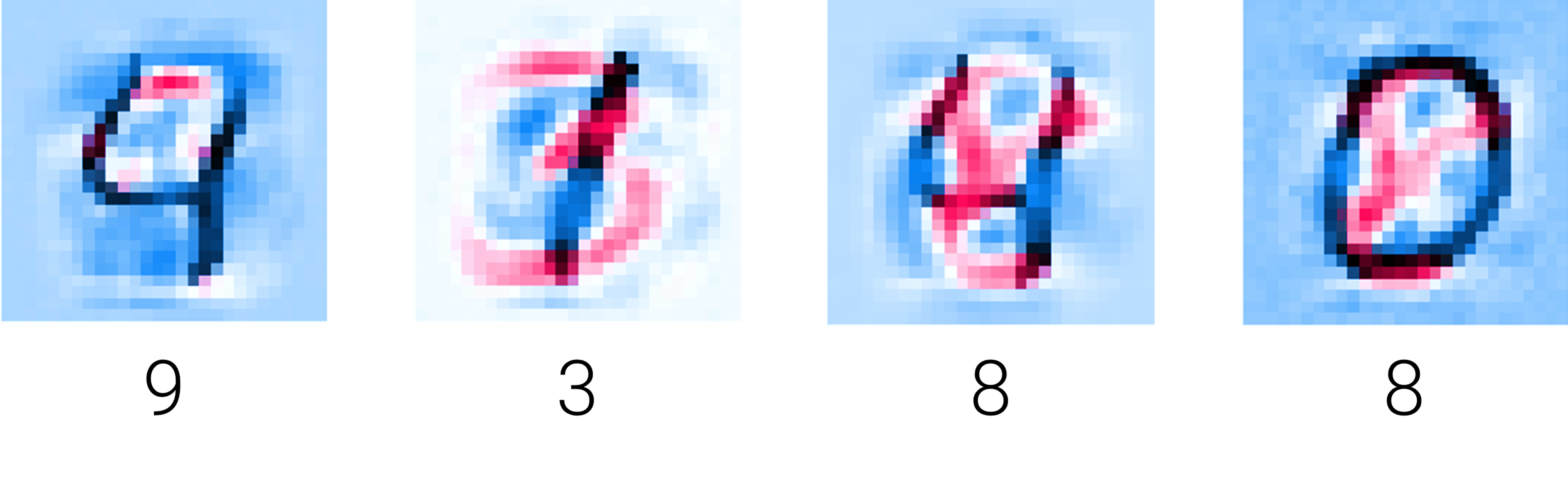}\vspace{-5mm}
  \caption{\textbf{Targeted explanations.} Generated explanations for a decision other than the one predicted by the model. The class explained is indicated at the bottom of each sample, e.g., the first sample is a `4' and the explanation is for the class `9'. As indicated in section~\ref{sec:targeted_explanations}, the red areas indicate that a black line should be added and the blue areas that it should be removed. More examples are available in the Appendix.
  }
  \vspace{-4mm}
  \label{fig:targeted_explanations}
\end{figure}

In some cases, it is instructive to look at the explanations for unpredicted classes in order to get information about the internal mechanisms of the models studied. Such explanations allow us to highlight contrastive features: elements that should be changed or whose absence is critical. 
Our method allows us to obtain such explanations: for a given input, we are then exclusively interested in the class we are trying to explain, without looking at the other decisions. Formally, for a given targeted class $c'$
the \adv~ (Equation~\ref{eq:adv_surface}) become $\AO(\vx, \mathcal{B}) = \max_{\substack{\vd \in \mathcal{B}}} \pred_{c'}(\vx + \vd) - \pred_c(\vx + \vd)$. Moreover, by splitting the perturbation ball into a positive one $\ball^{(+)} = \big\{ \vd \in \ball ~:~ \vd_i \geq 0,~ \forall i \in \{1, ..., d \} \big\}$ and a negative one $\ball^{(-)} = \big\{ \vd \in \ball ~:~ \vd_i \leq 0,~ \forall i \in \{1, ..., d \} \big\}$, one can deduce which direction -- adding or removing the black line in the case of gray-scaled images -- will impact the most the model decision.

We generate targeted explanations on the MNIST dataset using (IBP+Forward+Backward). For several inputs, we generate the explanation for the 10 classes. Fig.~\ref{fig:ap_targeted} shows 4 examples of targeted explanations, the target class $c'$ is indicated at the bottom. The red areas indicate that adding a black line increases the \adv~ with the target class. Conversely, the blue areas indicate where the increase of the score requires removing black lines.
All other results can be found in the Appendix.
In addition to favorable results on the fidelity metrics and guarantees provided by the verification methods, \eva~can provide targeted explanations that are easily understandable by humans, which are two qualities that make them a candidate of choice to meet the recent General Data Protection Regulation (GDPR) adopted in Europe~\cite{kaminski2021right}. More examples are available in the Appendix~\ref{ap:targeted}.
\vspace{-2.2mm}
\section{Conclusion}

In this work, we presented the first explainability method that uses verification perturbation analysis that exhaustively explores the perturbation space to generate explanations. We presented an efficient estimator that yields explanations that are state-of-the-art on current metrics. We also described a simple strategy to scale up perturbation verification methods to complex models. Finally, we showed that this estimator can be used to form easily interpretable targeted explanations.

We hope that this work will %
for searching for safer and more efficient explanation methods for neural networks -- and that it will inspire further synergies with the field of formal verification.

\newpage
\balance
{\small
\bibliographystyle{ieee_fullname}
\bibliography{egbib}

\begin{thebibliography}{10}\itemsep=-1pt

\bibitem{tensorflow2015}
Mart\'{\i}n Abadi, Ashish Agarwal, Paul Barham, Eugene Brevdo, Zhifeng Chen,
  Craig Citro, Greg~S. Corrado, Andy Davis, Jeffrey Dean, Matthieu Devin,
  Sanjay Ghemawat, Ian Goodfellow, Andrew Harp, Geoffrey Irving, Michael Isard,
  Yangqing Jia, Rafal Jozefowicz, Lukasz Kaiser, Manjunath Kudlur, Josh
  Levenberg, Dandelion Man\'{e}, Rajat Monga, Sherry Moore, Derek Murray, Chris
  Olah, Mike Schuster, Jonathon Shlens, Benoit Steiner, Ilya Sutskever, Kunal
  Talwar, Paul Tucker, Vincent Vanhoucke, Vijay Vasudevan, Fernanda Vi\'{e}gas,
  Pete Warden, Martin Wattenberg, Martin Wicke, Yuan Yu, and Xiaoqiang Zheng.
\newblock {TensorFlow}: Large-scale machine learning on heterogeneous systems,
  2015.

\bibitem{adebayo2018sanity}
Julius Adebayo, Justin Gilmer, Michael Muelly, Ian Goodfellow, Moritz Hardt,
  and Been Kim.
\newblock Sanity checks for saliency maps.
\newblock In {\em Advances in Neural Information Processing Systems (NIPS)},
  2018.

\bibitem{ancona2017better}
Marco Ancona, Enea Ceolini, Cengiz Öztireli, and Markus Gross.
\newblock Towards better understanding of gradient-based attribution methods
  for deep neural networks.
\newblock In {\em Proceedings of the International Conference on Learning
  Representations (ICLR)}, 2018.

\bibitem{baehrens2010explain}
David Baehrens, Timon Schroeter, Stefan Harmeling, Motoaki Kawanabe, Katja
  Hansen, and Klaus-Robert M{\"u}ller.
\newblock How to explain individual classification decisions.
\newblock {\em The Journal of Machine Learning Research}, 11:1803--1831, 2010.

\bibitem{balunovic2019adversarial}
Mislav Balunovic and Martin Vechev.
\newblock Adversarial training and provable defenses: Bridging the gap.
\newblock In {\em Proceedings of the International Conference on Learning
  Representations (ICLR)}, 2019.

\bibitem{bassan2022towards}
Shahaf Bassan and Guy Katz.
\newblock Towards formal approximated minimal explanations of neural networks.
\newblock {\em arXiv preprint arXiv:2210.13915}, 2022.

\bibitem{aggregating2020}
Umang Bhatt, Adrian Weller, and José M.~F. Moura.
\newblock Evaluating and aggregating feature-based model explanations.
\newblock In {\em Proceedings of the International Joint Conference on
  Artificial Intelligence (IJCAI)}, 2020.

\bibitem{boopathy2020proper}
Akhilan Boopathy, Sijia Liu, Gaoyuan Zhang, Cynthia Liu, Pin-Yu Chen, Shiyu
  Chang, and Luca Daniel.
\newblock Proper network interpretability helps adversarial robustness in
  classification.
\newblock In {\em Proceedings of the International Conference on Machine
  Learning (ICML)}, 2020.

\bibitem{boumazouza2021asteryx}
Ryma Boumazouza, Fahima Cheikh-Alili, Bertrand Mazure, and Karim Tabia.
\newblock Asteryx: A model-agnostic sat-based approach for symbolic and
  score-based explanations.
\newblock In {\em Proceedings of the 30th ACM International Conference on
  Information \& Knowledge Management}, pages 120--129, 2021.

\bibitem{chattopadhay2018grad}
Aditya Chattopadhay, Anirban Sarkar, Prantik Howlader, and Vineeth~N
  Balasubramanian.
\newblock Grad-cam++: Generalized gradient-based visual explanations for deep
  convolutional networks.
\newblock In {\em Proceedings of the IEEE/CVF Winter Conference on Applications
  of Computer Vision (WACV)}, 2018.

\bibitem{chollet2015keras}
Fran\c{c}ois Chollet et~al.
\newblock Keras.
\newblock \url{https://keras.io}, 2015.

\bibitem{fel2021cannot}
Julien Colin, Thomas Fel, R{\'e}mi Cad{\`e}ne, and Thomas Serre.
\newblock What i cannot predict, i do not understand: A human-centered
  evaluation framework for explainability methods.
\newblock {\em Advances in Neural Information Processing Systems (NeurIPS)},
  2021.

\bibitem{imagenet_cvpr09}
J. Deng, W. Dong, R. Socher, L.-J. Li, K. Li, and L. Fei-Fei.
\newblock {ImageNet: A Large-Scale Hierarchical Image Database}.
\newblock In {\em Proceedings of the IEEE Conference on Computer Vision and
  Pattern Recognition (CVPR)}, 2009.

\bibitem{doshivelez2017rigorous}
Finale Doshi-Velez and Been Kim.
\newblock Towards a rigorous science of interpretable machine learning.
\newblock {\em {A}r{X}iv e-print}, 2017.

\bibitem{ducoffe2021decomon}
{Ducoffe, Melanie}.
\newblock Decomon: Automatic certified perturbation analysis of neural
  networks, 2021.

\bibitem{ehlers2017formal}
Ruediger Ehlers.
\newblock Formal verification of piece-wise linear feed-forward neural
  networks.
\newblock In {\em International Symposium on Automated Technology for
  Verification and Analysis}, pages 269--286. Springer, 2017.

\bibitem{fel2021sobol}
Thomas Fel, Remi Cadene, Mathieu Chalvidal, Matthieu Cord, David Vigouroux, and
  Thomas Serre.
\newblock Look at the variance! efficient black-box explanations with
  sobol-based sensitivity analysis.
\newblock In {\em Advances in Neural Information Processing Systems (NeurIPS)},
  2021.

\bibitem{fel2021xplique}
Thomas Fel, Lucas Hervier, David Vigouroux, Antonin Poche, Justin Plakoo, Remi
  Cadene, Mathieu Chalvidal, Julien Colin, Thibaut Boissin, Louis Béthune,
  Agustin Picard, Claire Nicodeme, Laurent Gardes, Gregory Flandin, and Thomas
  Serre.
\newblock Xplique: A deep learning explainability toolbox.
\newblock {\em Workshop, Proceedings of the IEEE Conference on Computer Vision
  and Pattern Recognition (CVPR)}, 2022.

\bibitem{fel2020representativity}
Thomas Fel and David Vigouroux.
\newblock Representativity and consistency measures for deep neural network
  explanations.
\newblock In {\em Proceedings of the IEEE/CVF Winter Conference on Applications
  of Computer Vision (WACV)}, 2022.

\bibitem{ferrettini2021coalitional}
Gabriel Ferrettini, Elodie Escriva, Julien Aligon, Jean-Baptiste Excoffier, and
  Chantal Soul{\'e}-Dupuy.
\newblock Coalitional strategies for efficient individual prediction
  explanation.
\newblock {\em Information Systems Frontiers}, pages 1--27, 2021.

\bibitem{Fong_2017}
Ruth~C. Fong and Andrea Vedaldi.
\newblock Interpretable explanations of black boxes by meaningful perturbation.
\newblock In {\em Proceedings of the IEEE International Conference on Computer
  Vision (ICCV)}, 2017.

\bibitem{ghalebikesabi2021locality}
Sahra Ghalebikesabi, Lucile Ter-Minassian, Karla DiazOrdaz, and Chris~C Holmes.
\newblock On locality of local explanation models.
\newblock {\em Advances in Neural Information Processing Systems (NeurIPS)},
  2021.

\bibitem{ghorbani2017interpretation}
Amirata Ghorbani, Abubakar Abid, and James Zou.
\newblock Interpretation of neural networks is fragile.
\newblock In {\em Proceedings of the AAAI Conference on Artificial Intelligence
  (AAAI)}, 2017.

\bibitem{graziani2021sharpening}
Mara Graziani, Iam Palatnik~de Sousa, Marley~MBR Vellasco, Eduardo Costa~da
  Silva, Henning M{\"u}ller, and Vincent Andrearczyk.
\newblock Sharpening local interpretable model-agnostic explanations for
  histopathology: improved understandability and reliability.
\newblock In {\em Medical Image Computing and Computer Assisted Intervention
  (MICCAI)}. Springer, 2021.

\bibitem{hase2021out}
Peter Hase, Harry Xie, and Mohit Bansal.
\newblock The out-of-distribution problem in explainability and search methods
  for feature importance explanations.
\newblock {\em Advances in Neural Information Processing Systems (NeurIPS)},
  2021.

\bibitem{haug2021baselines}
Johannes Haug, Stefan Z{\"u}rn, Peter El-Jiz, and Gjergji Kasneci.
\newblock On baselines for local feature attributions.
\newblock {\em arXiv preprint arXiv:2101.00905}, 2021.

\bibitem{hedstrom2022quantus}
Anna Hedstr{\"o}m, Leander Weber, Dilyara Bareeva, Franz Motzkus, Wojciech
  Samek, Sebastian Lapuschkin, and Marina M-C H{\"o}hne.
\newblock Quantus: an explainable ai toolkit for responsible evaluation of
  neural network explanations.
\newblock {\em The Journal of Machine Learning Research (JMLR)}, 2022.

\bibitem{hooker2018benchmark}
Sara Hooker, Dumitru Erhan, Pieter-Jan Kindermans, and Been Kim.
\newblock A benchmark for interpretability methods in deep neural networks.
\newblock In {\em Advances in Neural Information Processing Systems (NeurIPS)},
  2019.

\bibitem{hsieh2020evaluations}
Cheng-Yu Hsieh, Chih-Kuan Yeh, Xuanqing Liu, Pradeep Ravikumar, Seungyeon Kim,
  Sanjiv Kumar, and Cho-Jui Hsieh.
\newblock Evaluations and methods for explanation through robustness analysis.
\newblock In {\em Proceedings of the International Conference on Learning
  Representations (ICLR)}, 2021.

\bibitem{idrissi2023coalitional}
Marouane~Il Idrissi, Nicolas Bousquet, Fabrice Gamboa, Bertrand Iooss, and
  Jean-Michel Loubes.
\newblock On the coalitional decomposition of parameters of interest, 2023.

\bibitem{idrissi2021developments}
Marouane~Il Idrissi, Vincent Chabridon, and Bertrand Iooss.
\newblock Developments and applications of shapley effects to
  reliability-oriented sensitivity analysis with correlated inputs.
\newblock {\em Environmental Modelling \& Software}, 2021.

\bibitem{ignatiev2019abduction}
Alexey Ignatiev, Nina Narodytska, and Joao Marques-Silva.
\newblock Abduction-based explanations for machine learning models.
\newblock In {\em Advances in Neural Information Processing Systems (NeurIPS)},
  2019.

\bibitem{ignatiev2019relating}
Alexey Ignatiev, Nina Narodytska, and Joao Marques-Silva.
\newblock On relating explanations and adversarial examples.
\newblock In {\em Advances in Neural Information Processing Systems (NeurIPS)},
  2019.

\bibitem{jacovi2020towards}
Alon Jacovi and Yoav Goldberg.
\newblock Towards faithfully interpretable nlp systems: How should we define
  and evaluate faithfulness?
\newblock {\em Proceedings of the Annual Meeting of the Association for
  Computational Linguistics (ACL Short Papers)}, 2020.

\bibitem{kaminski2021right}
Margot~E Kaminski.
\newblock The right to explanation, explained.
\newblock In {\em Research Handbook on Information Law and Governance}. Edward
  Elgar Publishing, 2021.

\bibitem{katz2017reluplex}
Guy Katz, Clark Barrett, David~L Dill, Kyle Julian, and Mykel~J Kochenderfer.
\newblock Reluplex: An efficient smt solver for verifying deep neural networks.
\newblock In {\em International Conference on Computer Aided Verification},
  pages 97--117. Springer, 2017.

\bibitem{kim2021hive}
Sunnie S.~Y. Kim, Nicole Meister, Vikram~V. Ramaswamy, Ruth Fong, and Olga
  Russakovsky.
\newblock {HIVE}: Evaluating the human interpretability of visual explanations.
\newblock In {\em Proceedings of the IEEE European Conference on Computer
  Vision (ECCV)}, 2022.

\bibitem{kindermans2019reliability}
Pieter-Jan Kindermans, Sara Hooker, Julius Adebayo, Maximilian Alber, Kristof~T
  Sch{\"u}tt, Sven D{\"a}hne, Dumitru Erhan, and Been Kim.
\newblock The (un) reliability of saliency methods.
\newblock 2019.

\bibitem{krizhevsky2009learning}
Alex Krizhevsky, Geoffrey Hinton, et~al.
\newblock Learning multiple layers of features from tiny images, 2009.

\bibitem{lecun2015deep}
Yann LeCun, Yoshua Bengio, and Geoffrey Hinton.
\newblock Deep learning.
\newblock {\em Nature}, 2015.

\bibitem{lecun2010mnist}
Yann LeCun and Corinna Cortes.
\newblock {MNIST} handwritten digit database, 2010.

\bibitem{lin2019explanations}
Zhong~Qiu Lin, Mohammad~Javad Shafiee, Stanislav Bochkarev, Michael~St Jules,
  Xiao~Yu Wang, and Alexander Wong.
\newblock Do explanations reflect decisions? a machine-centric strategy to
  quantify the performance of explainability algorithms.
\newblock In {\em Advances in Neural Information Processing Systems (NIPS)},
  2019.

\bibitem{lundberg2017unified}
Scott Lundberg and Su-In Lee.
\newblock A unified approach to interpreting model predictions.
\newblock In {\em Advances in Neural Information Processing Systems (NIPS)},
  2017.

\bibitem{madry2017pgd}
Aleksander Madry, Aleksandar Makelov, Ludwig Schmidt, Dimitris Tsipras, and
  Adrian Vladu.
\newblock Towards deep learning models resistant to adversarial attacks.
\newblock {\em Proceedings of the International Conference on Learning
  Representations (ICLR)}, 2018.

\bibitem{nguyen2021effectiveness}
Giang Nguyen, Daeyoung Kim, and Anh Nguyen.
\newblock The effectiveness of feature attribution methods and its correlation
  with automatic evaluation scores.
\newblock {\em Advances in Neural Information Processing Systems (NeurIPS)},
  2021.

\bibitem{novello2022making}
Paul Novello, Thomas Fel, and David Vigouroux.
\newblock Making sense of dependence: Efficient black-box explanations using
  dependence measure.
\newblock In {\em Advances in Neural Information Processing Systems (NeurIPS)},
  2022.

\bibitem{o2020generative}
Matthew O'Shaughnessy, Gregory Canal, Marissa Connor, Mark Davenport, and
  Christopher Rozell.
\newblock Generative causal explanations of black-box classifiers.
\newblock In {\em Advances in Neural Information Processing Systems (NeurIPS)},
  2020.

\bibitem{petsiuk2018rise}
Vitali Petsiuk, Abir Das, and Kate Saenko.
\newblock Rise: Randomized input sampling for explanation of black-box models.
\newblock In {\em Proceedings of the British Machine Vision Conference (BMVC)},
  2018.

\bibitem{ribeiro2016i}
Marco~Tulio Ribeiro, Sameer Singh, and Carlos Guestrin.
\newblock "why should i trust you?": Explaining the predictions of any
  classifier.
\newblock In {\em Knowledge Discovery and Data Mining (KDD)}, 2016.

\bibitem{ross2021learning}
Alexis Ross, Himabindu Lakkaraju, and Osbert Bastani.
\newblock Learning models for actionable recourse.
\newblock {\em Advances in Neural Information Processing Systems (NeurIPS)},
  2021.

\bibitem{salman2019convex}
Hadi Salman, Greg Yang, Huan Zhang, Cho-Jui Hsieh, and Pengchuan Zhang.
\newblock A convex relaxation barrier to tight robustness verification of
  neural networks.
\newblock {\em Advances in Neural Information Processing Systems (NeurIPS)},
  2019.

\bibitem{samek2016evaluating}
Wojciech Samek, Alexander Binder, Gr{\'e}goire Montavon, Sebastian Lapuschkin,
  and Klaus-Robert M{\"u}ller.
\newblock Evaluating the visualization of what a deep neural network has
  learned.
\newblock {\em IEEE transactions on neural networks and learning systems},
  2016.

\bibitem{Selvaraju_2019}
Ramprasaath~R. Selvaraju, Michael Cogswell, Abhishek Das, Ramakrishna Vedantam,
  Devi Parikh, and Dhruv Batra.
\newblock Grad-cam: Visual explanations from deep networks via gradient-based
  localization.
\newblock In {\em Proceedings of the IEEE International Conference on Computer
  Vision (ICCV)}, 2017.

\bibitem{seo2018noise}
Junghoon Seo, Jeongyeol Choe, Jamyoung Koo, Seunghyeon Jeon, Beomsu Kim, and
  Taegyun Jeon.
\newblock Noise-adding methods of saliency map as series of higher order
  partial derivative.
\newblock In {\em Workshop on Human Interpretability in Machine Learning,
  Proceedings of the International Conference on Machine Learning (ICML)},
  2018.

\bibitem{shrikumar2017learning}
Avanti Shrikumar, Peyton Greenside, and Anshul Kundaje.
\newblock Learning important features through propagating activation
  differences.
\newblock In {\em Proceedings of the International Conference on Machine
  Learning (ICML)}, 2017.

\bibitem{simonyan2013deep}
Karen Simonyan, Andrea Vedaldi, and Andrew Zisserman.
\newblock Deep inside convolutional networks: Visualising image classification
  models and saliency maps.
\newblock In {\em Workshop, Proceedings of the International Conference on
  Learning Representations (ICLR)}, 2013.

\bibitem{simonyan2014deep}
Karen Simonyan, Andrea Vedaldi, and Andrew Zisserman.
\newblock Deep inside convolutional networks: Visualising image classification
  models and saliency maps.
\newblock In {\em Workshop Proceedings of the International Conference on
  Learning Representations (ICLR)}, 2014.

\bibitem{singh2019abstract}
Gagandeep Singh, Timon Gehr, Markus P{\"u}schel, and Martin Vechev.
\newblock An abstract domain for certifying neural networks.
\newblock {\em Proceedings of the ACM on Programming Languages}, 2019.

\bibitem{slack2021counterfactual}
Dylan Slack, Anna Hilgard, Himabindu Lakkaraju, and Sameer Singh.
\newblock Counterfactual explanations can be manipulated.
\newblock {\em Advances in Neural Information Processing Systems (NeurIPS)},
  2021.

\bibitem{slack2021reliable}
Dylan Slack, Anna Hilgard, Sameer Singh, and Himabindu Lakkaraju.
\newblock Reliable post hoc explanations: Modeling uncertainty in
  explainability.
\newblock {\em Advances in Neural Information Processing Systems (NeurIPS)},
  34, 2021.

\bibitem{smilkov2017smoothgrad}
Daniel Smilkov, Nikhil Thorat, Been Kim, Fernanda Viégas, and Martin
  Wattenberg.
\newblock Smoothgrad: removing noise by adding noise.
\newblock In {\em Workshop on Visualization for Deep Learning, Proceedings of
  the International Conference on Machine Learning (ICML)}, 2017.

\bibitem{sotoudeh2019computing}
Matthew Sotoudeh and Aditya~V. Thakur.
\newblock Computing linear restrictions of neural networks.
\newblock In {\em Advances in Neural Information Processing Systems (NeurIPS)},
  2019.

\bibitem{springenberg2014striving}
Jost~Tobias Springenberg, Alexey Dosovitskiy, Thomas Brox, and Martin
  Riedmiller.
\newblock Striving for simplicity: The all convolutional net.
\newblock In {\em Workshop Proceedings of the International Conference on
  Learning Representations (ICLR)}, 2014.

\bibitem{sturmfels2020visualizing}
Pascal Sturmfels, Scott Lundberg, and Su-In Lee.
\newblock Visualizing the impact of feature attribution baselines.
\newblock {\em Distill}, 2020.

\bibitem{sundararajan2017axiomatic}
Mukund Sundararajan, Ankur Taly, and Qiqi Yan.
\newblock Axiomatic attribution for deep networks.
\newblock In {\em Proceedings of the International Conference on Machine
  Learning (ICML)}, 2017.

\bibitem{tjeng2017verifying}
Vincent Tjeng and Russ Tedrake.
\newblock Verifying neural networks with mixed integer programming.
\newblock {\em Proceedings of the International Conference on Learning
  Representations (ICLR)}, 15, 2019.

\bibitem{wang2021beta}
Shiqi Wang, Huan Zhang, Kaidi Xu, Xue Lin, Suman Jana, Cho-Jui Hsieh, and
  J~Zico Kolter.
\newblock Beta-crown: Efficient bound propagation with per-neuron split
  constraints for neural network robustness verification.
\newblock {\em Advances in Neural Information Processing Systems (NeurIPS)},
  2021.

\bibitem{xu2020automatic}
Kaidi Xu, Zhouxing Shi, Huan Zhang, Yihan Wang, Kai-Wei Chang, Minlie Huang,
  Bhavya Kailkhura, Xue Lin, and Cho-Jui Hsieh.
\newblock Automatic perturbation analysis for scalable certified robustness and
  beyond.
\newblock {\em Advances in Neural Information Processing Systems (NeurIPS)},
  2020.

\bibitem{yeh2019infidelity}
Chih-Kuan Yeh, Cheng-Yu Hsieh, Arun~Sai Suggala, David~I. Inouye, and Pradeep
  Ravikumar.
\newblock On the (in)fidelity and sensitivity for explanations.
\newblock In {\em Advances in Neural Information Processing Systems (NeurIPS)},
  2019.

\bibitem{yin2022sensitivity}
Fan Yin, Zhouxing Shi, Cho-Jui Hsieh, and Kai-Wei Chang.
\newblock On the sensitivity and stability of model interpretations in nlp.
\newblock In {\em Proceedings of the 60th Annual Meeting of the Association for
  Computational Linguistics (Volume 1: Long Papers)}, pages 2631--2647, 2022.

\bibitem{zeiler2013visualizing}
Matthew~D Zeiler and Rob Fergus.
\newblock Visualizing and understanding convolutional networks.
\newblock In {\em Proceedings of the IEEE European Conference on Computer
  Vision (ECCV)}, 2014.

\bibitem{zeiler2014visualizing}
Matthew~D Zeiler and Rob Fergus.
\newblock Visualizing and understanding convolutional networks.
\newblock In {\em Proceedings of the IEEE European Conference on Computer
  Vision (ECCV)}, 2014.

\bibitem{zhang2018efficient}
Huan Zhang, Tsui-Wei Weng, Pin-Yu Chen, Cho-Jui Hsieh, and Luca Daniel.
\newblock Efficient neural network robustness certification with general
  activation functions.
\newblock {\em Advances in Neural Information Processing Systems (NeurIPS)},
  2018.

\end{thebibliography}
}

\appendix
\clearpage

\section{Qualitative comparison}

Regarding the visual consistency of our method, Figure~\ref{fig:imagenet_explanations} shows a side-by-side comparison between our method and the attribution methods  tested in our benchmark. 
To allow better visualization, the gradient-based
methods were 2 percentile clipped.

\begin{figure*}[!ht]
  \centering
  \includegraphics[width=0.99\textwidth]{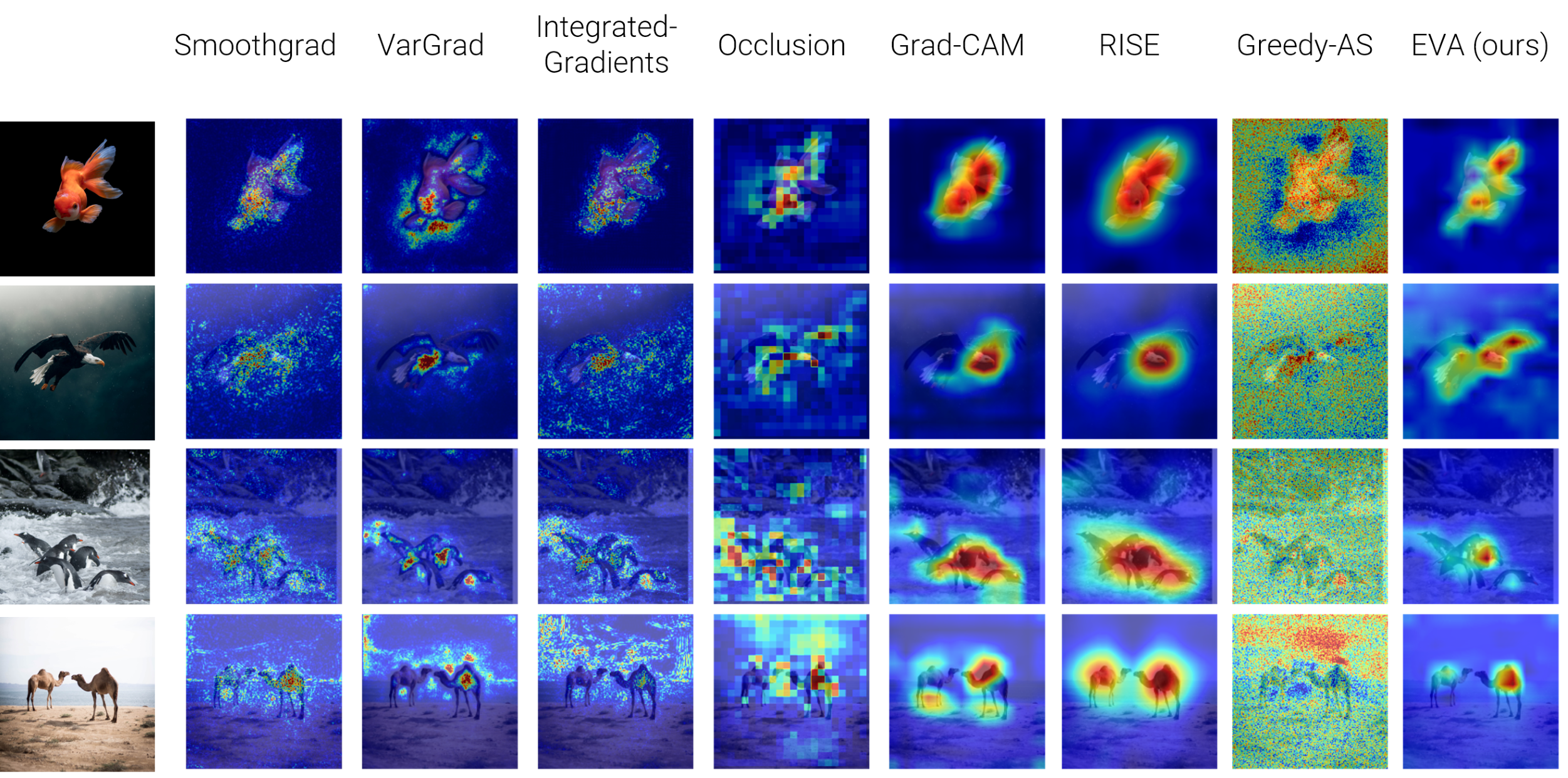}
  \caption{\textbf{Qualitative comparison} with other attribution methods. To allow for better visualization, the gradient-based methods (Saliency, Gradient-Input, SmoothGrad, Integrated-Gradient, VarGrad) are clipped at the 2nd percentile. For more results and details on each method and choice of hyperparameters, see Appendix.
  }
  \vspace{-2mm}
  \label{fig:imagenet_explanations}
\end{figure*}

\section{Ablation studies}
\label{ap:benchmarks}

\begin{table*}[t]
    \centering
    \scalebox{0.92}{
        \begin{tabular}{l C{0mm} P{0mm} P{5mm}P{5mm}P{5mm}P{5mm}P{5mm} P{1mm} P{5mm}P{5mm}P{9mm}P{5mm}P{5mm} P{1mm} P{5mm}P{5mm}P{5mm}P{5mm}P{5mm}}
        \toprule
        &&& \multicolumn{5}{c}{MNIST} &&  \multicolumn{5}{c}{Cifar-10} &&  \multicolumn{5}{c}{ImageNet}  \\
        \cmidrule(lr){4-8} \cmidrule(lr){10-14} \cmidrule(lr){16-20}
        
        &&& Del.$\downarrow$ & Ins.$\uparrow$ & Fid.$\uparrow$ & Rob.$\downarrow$ & Time 
        && Del.$\downarrow$ & Ins.$\uparrow$ & Fid.$\uparrow$ & Rob.$\downarrow$ & Time 
        && Del.$\downarrow$ & Ins.$\uparrow$ & Fid.$\uparrow$ & Rob.$\downarrow$ & Time
        \\
        
        \midrule
        Greedy-AS\cite{hsieh2020evaluations} &   && .260  & .497  & .110  & \textbf{.061}  & 335
                 && .205 & .264 & -.003 & \textbf{.013} & 4618
                 && .088 & .047 & .023 & \textbf{.612} & 180056  \\
        \midrule 
        Greedy-AO &&& .237 & .572 & .244 & \underline{.063} & 290
                 && \textbf{.162} & .283 & .041 & .024 & 2874
                 && .086 & .050 & .023 & \underline{.752} & 26762 \\
        \evaEmp  &&& \underline{.101} & .621 & \underline{.378} & .067 & 14.4 
                 && .184 & .270 & \textbf{.397} & \underline{.022} & 186.6
                 && \multirow{ 2}{*}{.070} & \multirow{ 2}{*}{\textbf{.289}} & \multirow{ 2}{*}{.048} & \multirow{ 2}{*}{.758} & \multirow{ 2}{*}{6454} \\ 
        \textbf{\eva} ~(ours)     & && \textbf{.089} & \textbf{.736} & \textbf{.428} & \underline{.069} & 1.29 
                 && \textbf{.164} & \underline{.290} & \textbf{.352} & \underline{.025} & 12.7 
                 \\
        
        \bottomrule \\
        \end{tabular}
    }
    \caption{
Results on Deletion (Del.), Insertion (Ins.), $\mu$Fidelity (Fid.) and \rsr~ (Rob.) metrics. 
Time in seconds corresponds to the generation of 100 explanations on an Nvidia P100.
Note that \eva~is the only method with guarantees that the entire set of possible perturbations has been exhaustively searched.
Verified perturbation analysis with IBP + Forward + Backward is used for MNIST, with Forward only for  CIFAR-10 and with our hybrid strategy described in Section.\ref{sec:scaling} for ImageNet. 
Grad-CAM and Grad-CAM++ are not calculated on the MNIST dataset since the network  only has dense layers. 
Greedy-AO is the equivalent of Greedy-AS but with the \AO estimator. 
The first and second best results are  in \textbf{bold} and \underline{underlined}, respectively. 
}
    \label{tab:ablation_ao}
    \vspace{-3mm}
\end{table*}

For a more thorough understanding of the impact of the different components that made EVA - the adversarial overlap and the use of verification tools- we proposed different ablation versions of EVA which are the following:
(\textbf{\textit{i}}) Empirical EVA, (\textbf{\textit{ii}}) GreedyAO which is the equivalent of Greedy-AS but with the $\AO$ estimator. This allow us to perform ablation on the proposed $\AO$~estimator. Results can be found in Table~\ref{tab:ablation_ao}.
\subsection{Empirical EVA.}
    
In this section, we describe the ablation consisting in estimating \eva~ without any use of verified perturbation analysis -- thus without any guarantees.

A first intuitive approach would be to replace verification perturbation analysis with adversarial attacks (as used in \textit{Greedy-AS}~\cite{hsieh2020evaluations}); we denote this approach as \textit{Greedy-AO}. 
In addition, we go further with a purely statistical approach based on a uniform sampling of the domain; we denote this approach \evaEmp. 
    
This estimator proves to be a very good alternative in terms of computation time but also with respect to the considered metrics as shown in Section ~\ref{sec:experiments}. Unfortunately the lack of guarantee makes it not as relevant as \eva.
Formally, it consists in directly estimating empirically \AO  using $N$ randomly sampled perturbations.
    
\begin{equation}\label{eq:adv_empirique}
    \AOemp(\vx, \mathcal{B}) = 
    \max_{\substack{\vd_1,\cdots \vd_i,\cdots \vd_N  \overset{\mathrm{iid}}{\sim} U(\mathcal{B})\\c'\neq{}c}} \pred_{c'}(\vx + \vd_i) - \pred_c(\vx + \vd_i).
\end{equation}
    
We then denote accordingly \evaEmp which uses $\AOemp$:

\begin{equation}
    \label{eq:tod_estimator_emp}
    \evaEmp(\vx, \vu, \ball) = \AOemp(\vx, \ball) - \AOemp(\vx, \ballu)
\end{equation}

\section{\eva~and $\rsr$}

We show here that the explanations generated by \eva~  provide an optimal solution from a certain stage to the $\rsr$ metric proposed by~\cite{hsieh2020evaluations}. We admit a unique closest adversarial perturbation $\vd^* = \min ||\vd||_p : \pred(\vx + \vd) \neq \pred(\vx)$, and we define $\varepsilon$, the radius of $\ball$ as $\varepsilon = ||\vd||_p$. 
Note that $||\vd||_p$ can be obtained by binary search using the verified perturbation analysis method.

We briefly recall the $\rsr$ metric. With $\vx = (x_1, ..., x_d)$, the set $\mathcal{U} = \{1, ..., d\}$, $\vu$ a subset of $\mathcal{U}$ : $\vu \subseteq \mathcal{U}$ and $\overline{\vu}$ its complementary. Moreover, we denote the minimum distance to an adversarial example $\varepsilon^*_{\vu}$: 
$$ \varepsilon^*_{\vu} = \big\{ \min || \vd ||_p ~:~ \pred(\vx + \vd) \neq \pred(\vx), \vd_{\overline{\vu}} = 0  \big\} $$ 

The $\rsr$ score is the AUC of the curve formed by the points $\{ (1,  \varepsilon^{*}_{(1)}), ..., (d,  \varepsilon^{*}_{(d)})  \}$ where $\varepsilon^{*}_{(k)}$ is the minimum distance to an adversarial example for the $k$ most important variables.
From this, we can deduce that $||\vd^*|| \leq \varepsilon^*_{\vu}$, $\forall \vu \subseteq \{1, ..., d\}$.

The goal here is to minimize this score, which means for a number of variables $|\vu| = k$, finding the set of variables $\vu^*$ such that $\varepsilon^*_{\vu}$ is minimal. We call this set the \textit{optimal set at $k$}. 

\begin{definition}
The \textit{optimal set at $k$} is the set of variables $\vu^{*}_k $ such that 
$$ \vu^{*}_k = \underset{ \vu \subseteq \mathcal{U},~ |\vu| = k}{\argmin ~~ \varepsilon^*_{\vu} }. $$
\end{definition}

We note that finding the minimum cardinal of a variable to guarantee a decision is also a standard research problem  ~\cite{ignatiev2019abduction, ignatiev2019relating} and is called subset-minimal explanations. 

Intuitively, the optimal set is the combination of variables that allows finding the closest adversarial example.
Thus, minimizing $\rsr$ means finding the optimal set $\vu^*$ for each $k$. 
Note that this set can vary drastically from one step to another, it is therefore potentially impossible for attribution to satisfy this optimality criterion at each step.
Nevertheless, an optimal set that is always reached at some step is the one allowing to build $\vd^*$.
We start by defining the notion of an essential variable before showing the optimality of $\vd^*$.

\begin{definition}
Given an adversarial perturbation $\vd$, we call \textit{essentials variables} $\vu$ all variables such that $|\vd_{i}| > 0, i \in \vu$. Conversely, we call \textit{inessentials variables} variables that are not essential.
\end{definition}

For example, if $\vd^*$ has $k$ \textit{essential variables}, it is reachable by modifying only $k$ variables. 
This allows us to characterize the optimal set at step $k$.

\begin{proposition} 
\label{prop:uoptimal}
Let $\vu$ be the set of essential variables of $\vd^*$, then $\vu$ is an optimal set for $k$, with $k \in [\![|\vu|,d]\!] $.
\end{proposition}

\begin{proof}
Let $\vv$ be a set such that $ \varepsilon^*_{\vv} < \varepsilon^*_{\vu} $, then $ \varepsilon^*_{\vv} < || \vd^* || $ which is a contradiction.
\end{proof}

Specifically, as soon as we have the variables allowing us to build $\vd^*$, then we reach the minimum possible for $\rsr$.
We will now show that \eva~allows us to reach this in $|\vu|$ steps, with $|\vu| \leq d$ by showing (1) that $\vd^*$ \textit{essential variables} obtain a positive attribution and (2) that $\vd^*$ \textit{inessential variables} obtain a zero attribution.

\begin{proposition}
\label{prop:ess}
All essential variables \(\vu\) w.r.t \(\vd^*\) have a strictly positive importance score \(\eva(\vu) > 0\). 
\end{proposition}

\begin{proof}
Let us assume that $i$ is \textit{essential} and $\eva(i) = 0$, then $\bm{F}(\ball) = \bm{F}(\ball_i)$ which implies
$$
\max_{\substack{\vd \in \ball\\c'\neq{}c}} \pred_{c'}(\vx + \vd) - \pred_c(\vx + \vd) = 
\max_{\substack{\vd' \in \ball_i\\c'\neq{}c}} \pred_{c'}(\vx + \vd') - \pred_c(\vx + \vd')
$$
by uniqueness of the adversarial perturbation, $ \vd = \vd' $ which is a contradiction as $\vd' \notin \ball_i$ since $\vd'_i \neq 0$ by definition of an \textit{essential variable}. Thus $x_i$ cannot be \textit{essential}, which is a contradiction.
\end{proof}

Essentially, if the variable $i$ is necessary to reach $\vd^*$, then removing it prevents the adversarial example from being reached and lowers the \adv, giving a strictly positive attribution.

\begin{proposition}
\label{prop:iness}
All inessential variables \(\vv\) w.r.t. \(\vd^*\) have a zero importance score \(\eva(\vv) = 0\). 
\end{proposition}

\begin{proof}
With $i$ being an \textit{inessential} variable, then $\vd^*_i = 0$. It follow that $\vd^* \in \ball_i \subseteq \ball$. Thus
\begin{align*} 
\bm{F}(\ball) &= \max_{\substack{\vd \in \ball\\c'\neq{}c}} \pred_{c'}(\vx + \vd) - \pred_c(\vx + \vd) \\
              &= \pred_{c'}(\vx + \vd^*) - \pred_c(\vx + \vd^*)
\end{align*}
as $\vd^*$ is the unique adversarial perturbation in $\ball$, similarly 
\begin{align*} 
\bm{F}(\ball_i) &= \max_{\substack{\vd' \in \ball\\c'\neq{}c}} \pred_{c'}(\vx + \vd') - \pred_c(\vx + \vd') \\
              &= \pred_{c'}(\vx + \vd^*) - \pred_c(\vx + \vd^*)
\end{align*}
thus $\bm{F}(\ball) = \bm{F}(\ball_i)$ and $\eva(i) = 0$.
\end{proof}

Finally, since \eva~ranks the \textit{essential variables} of $\vd^*$ before the \textit{inessential variables}, and since $\vd^*$ is the \textit{optimal set} from the step $|\vu|$ to the last one $d$, then \eva~provide the \textit{optimal set}, at least from the step $|\vu|$.

\begin{theorem}\textbf{\eva~provide the optimal set from step $|\vu|$ to the last step.}
\label{thm:rsr}
With $\vu$ the essential variables of $\vd^*$, \eva~will rank the $\vu$ variables first and provide the optimal set from the step $|\vu|$ to the last step. 
\end{theorem}

\begin{proof}
Let $\vu$ denote the \textit{essential variables} of $\vd^*$ and $\vv$ the \textit{inessential variables}. Then according to Proposition~\ref{prop:ess} and Proposition~\ref{prop:iness}, $\forall i \in \vu, \forall j \in \vv: \eva(i) > \eva(j)$. It follow that $\vu$ are the most important variables at step $|\vu|$. Finally, according to Proposition~\ref{prop:uoptimal}, $\vu$ is the optimal set for $k$, with $k \in [\![|\vu|,d]\!]$.
\end{proof}

\begin{figure}[h!]
  \includegraphics[width=0.45\textwidth]{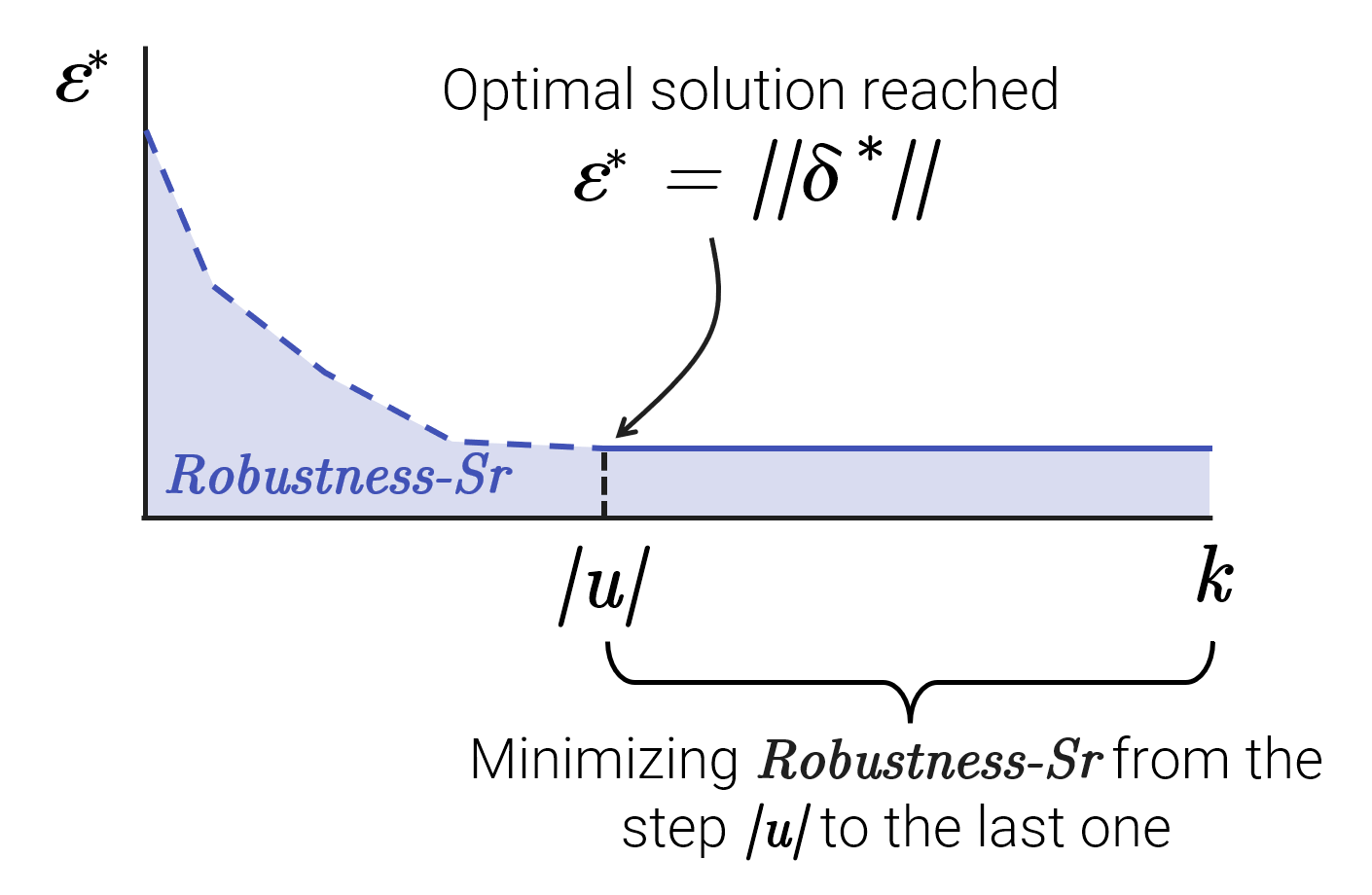}
  \caption{\textbf{\eva~yield optimal subset of variable from step $|\vu|$.} $\rsr$ measures the AUC of the distances to the nearest adversary for the $k$ most important variables. With $\vd^*$ the nearest reachable adversarial perturbation around $\vx$, then \eva~yield the optimal set -- the variables allowing to reach the nearest adversarial example for a given cardinality -- at least from $||\vu|| \leq d$ step to the last one, $\vu$ being the so-called essential variables.
  }
  \label{fig:eva_optimal}
\end{figure}

\section{\eva~and \textit{Stability}}

Stability is one of the most crucial properties of an explanation. Several metrics have been proposed~\cite{aggregating2020,yeh2019infidelity} and the most common one consists in finding around a point $\vx$, another point $\vz$ (in a radius $r$) such that the explanation changes the most according to a given distance between explanation $d$ and a distance over the inputs $\rho$:

$$
\textit{Stability}(\vx, \explainer) = \max_{\vz : \rho(\vz, \vx) \leq r} d( \explainer(\vx), \explainer(\vz) ) 
$$

and $\explainer$ an explanation functional.
It can be shown that the proposed ~\eva~ estimator is bounded by the stability of the model as well as by the radii $\varepsilon$ and $r$, $\varepsilon$ being the radius of $\ball$ and $r$ the radius of stability.
From here, we assume $d$ and $\rho$ are the $\ell_2$ distance.

Let assume that $\pred$ is $L$-lipschitz. We recall that a function $\pred$ is said $L$-lipschitz over $\mathcal{X}$ if and only if $\forall (\vx, \vz) \in \mathcal{X}^2, || \pred(\vx) - \pred(\vz) || \leq || \vx - \vz ||$.

\begin{theorem}\textbf{\eva~ has bounded Stability}
\label{thm:stab}
Given a $L$-lipschitz predictor $\pred$, $\varepsilon$ the radius of $\ball$ and $r$ the Stability radius, then
$$
\textit{Stability}(\vx, \eva) \leq 4L(\varepsilon + r)
$$
\end{theorem}

\begin{proof}
With $c' \neq c$ we denote $\vm(\vx) = \pred_{c'}(\vx) - \pred_{c}(\vx)$. We note that by additivity of the Lipschitz constant $\vm$ is 2$L$-Lipschitz.
\begin{align*} 
&\textit{Stability}(\vx, \eva) = \max_{\vz : \rho(\vz, \vx) \leq r} || \eva(\vx), \eva(\vz) || \\
  &= \max_{\vz : \rho(\vz, \vx) \leq r} 
  ||\max_{\vd} \vm(\vx + \vd) 
  - \max_{\vd_{\vu}} \vm(\vx + \vd_{\vu}) \\
  &~~~~- \max_{\vd} \vm(\vz + \vd)
  + \max_{\vd_{\vu}} \vm(\vz + \vd_{\vu}) || \\
  &\leq \max_{\vz : \rho(\vz, \vx) \leq r} 
  ||\max_{\vd} \vm(\vx + \vd) 
  - \max_{\vd} \vm(\vz + \vd) || \\
  &~~~~+ || \max_{\vd_{\vu}} \vm(\vz + \vd_{\vu}) 
  - \max_{\vd_{\vu}} \vm(\vx + \vd_{\vu}) || \\
  &= \max_{\vga : ||\vga|| \leq r} 
  ||\max_{\vd} \vm(\vx + \vd) 
  - \max_{\vd} \vm(\vx + \vd + \vga) || \\
  &~~~~+ || \max_{\vd_{\vu}} \vm(\vx + \vd_{\vu} + \vga) 
  - \max_{\vd_{\vu}} \vm(\vx + \vd_{\vu}) || \\
  &\leq 2L (||\vd|| + ||\vga||) + 2L (||\vd|| + ||\vga||)\\
  &= 4L (\varepsilon + r)
\end{align*}

\end{proof}

\section{Attribution methods}
\label{ap:methods}

In the following section, we give the formulation of the different attribution methods used in this work. The library used to generate the attribution maps is Xplique~\cite{fel2021xplique}.
By simplification of notation, we define $\pred(\bm{x})$ the logit score (before softmax) for the class of interest (we omit $c$). 
We recall that an attribution method provides an importance score for each input variable $x_i$.
We will denote the explanation functional mapping an input of interest $\vx = (x_1, ..., x_d) \in \mathcal{X}$ as $\explainer : \mathcal{X} \to \mathbb{R}^d$.

\textbf{Saliency}~\cite{simonyan2013deep} is a visualization technique based on the gradient of a class score relative to the input, indicating in an infinitesimal neighborhood, which pixels must be modified to most affect the score of the class of interest.

$$ \explainer(\bm{x}) = ||\nabla_{\bm{x}} \pred(\bm{x})|| $$

\textbf{Gradient $\odot$ Input}~\cite{shrikumar2017learning} is based on the gradient of a class score relative to the input, element-wise with the input, it was introduced to improve the sharpness of the attribution maps. A theoretical analysis conducted by~\cite{ancona2017better} showed that Gradient $\odot$ Input is equivalent to $\epsilon$-LRP and DeepLIFT~\cite{shrikumar2017learning} methods under certain conditions -- using a baseline of zero, and with all biases to zero.

$$ \explainer(\bm{x}) = \bm{x} \odot ||\nabla_{\bm{x}} \pred(\bm{x})|| $$

\textbf{Integrated Gradients}~\cite{sundararajan2017axiomatic} consists of summing the gradient values along the path from a baseline state to the current value. The baseline $\vx_0$ used is zero. This integral can be approximated with a set of $m$ points at regular intervals between the baseline and the point of interest. In order to approximate from a finite number of steps, we use a Trapezoidal rule and not a left-Riemann summation, which allows for more accurate results and improved performance (see~\cite{sotoudeh2019computing} for a comparison). For all the experiments $m = 100$.

$$ \explainer(\bm{x}) = (\bm{x} - \bm{x}_0) 
\int_0^1 \nabla_{\bm{x}} \pred( \bm{x}_0 + \alpha(\bm{x} - \bm{x}_0) )) d\alpha $$

\textbf{SmoothGrad}~\cite{smilkov2017smoothgrad} is also a gradient-based explanation method, which, as the name suggests, averages the gradient at several points corresponding to small perturbations (drawn i.i.d from an isotropic normal distribution of standard deviation $\sigma$) around the point of interest. The smoothing effect induced by the average help reducing the visual noise, and hence improve the explanations. The attribution is obtained by averaging after sampling $m$ points. For all the experiments, we took $m = 100$ and $\sigma = 0.2 \times (\vx_{\max} - \vx_{\min})$ where $(\vx_{\min}, \vx_{\max})$ being the input range of the dataset.

$$ \explainer(\bm{x}) = \underset{\vd \sim \mathcal{N}(0, \bm{I}\sigma)}{\mathbb{E}}(\nabla_{\bm{x}} \pred( \bm{x} + \vd) )
$$

\textbf{VarGrad}~\cite{hooker2018benchmark} is similar to SmoothGrad as it employs the same methodology to construct the attribution maps: using a set of $m$ noisy inputs, it aggregates the gradients using the variance rather than the mean. For the experiment, $m$ and $\sigma$ are the same as Smoothgrad. Formally:

$$ \explainer(\bm{x}) = \underset{\vd \sim \mathcal{N}(0, \bm{I}\sigma)}{\mathbb{V}}(\nabla_{\bm{x}} \pred( \bm{x} + \vd) )
$$

\textbf{Grad-CAM}~\cite{Selvaraju_2019} can only be used on Convolutional Neural Network (CNN). Thus we couldn't use it for the MNIST dataset. The method uses the gradient and the feature maps $\bm{A}^k$ of the last convolution layer. More precisely, to obtain the localization map for a class, we need to compute the weights $\alpha_c^k$ associated to each of the feature map activation $\bm{A}^k$, with $k$ the number of filters and $Z$ the number of features in each feature map, with $\alpha_k^c = \frac{1}{Z} \sum_i\sum_j \frac{\partial{\pred(\bm{x})}}{\partial \bm{A}^k_{ij}} $ and 

$$\explainer = \max(0, \sum_k \alpha_k^c \bm{A}^k) $$

As the size of the explanation depends on the size (width, height) of the last feature map, a bilinear interpolation is performed in order to find the same dimensions as the input. For all the experiments, we used the last convolutional layer of each model to compute the explanation.

\textbf{Grad-CAM++ (G+)}~\cite{chattopadhay2018grad} is an extension of Grad-CAM combining the
positive partial derivatives of feature maps of a convolutional layer with a weighted special class score. The weights $\alpha_c^{(k)}$ associated to each feature map is computed as follow : 

$$\alpha_k^c = 
    \sum_i \sum_j [
    \frac{ \frac{\partial^2 \pred(\vx) }{ (\partial \bm{A}_{ij}^{(k)})^2 } }
    { 2 \frac{\partial^2 \pred(\vx) }{ (\partial \bm{A}_{ij}^{(k)})^2 } + \sum_i \sum_j \bm{A}^{(k)}_{ij}  \frac{\partial^3 \pred(\vx) }{ (\partial \bm{A}_{ij}^{(k)})^3 } }
    ]
$$

\textbf{Occlusion}~\cite{zeiler2013visualizing} is a sensitivity method that sweeps a patch that occludes pixels over the images using a baseline state and use the variations of the model prediction to deduce critical areas. For all the experiments, we took a patch size and a patch stride of $\frac{1}{7}$ of the image size. Moreover, the baseline state $\vx_0$ was zero.

$$ \explainer(\vx)_i = \pred(\bm{x}) - \pred(\bm{x}_{[\bm{x}_i = 0]})  $$

\textbf{RISE}~\cite{petsiuk2018rise} is a black-box method that consists of probing the model with $N$ randomly masked versions of the input image to deduce the importance of each pixel using the corresponding outputs. The masks $\bm{m} \sim \mathcal{M}$ are generated randomly in a subspace of the input space. For all the experiments, we use a subspace of size $7 \times 7$, $N = 6000$, and $\mathbb{E}(\mathcal{M}) = 0.5$.

$$ \explainer(\bm{x}) = \frac{1}{\mathbb{E}(\mathcal{M}) N} \sum_{i=0}^N \pred(\bm{x} \odot \bm{m}_i) \bm{m}_i $$

\textbf{Greedy-AS}~\cite{hsieh2020evaluations}
is a greedy-like method which aggregates step by step the most important pixels -- the pixels that allow us to obtain the closest possible adversarial example. Starting from an empty set, we evaluate the importance of the variables at each step. Formally, with $\vu$ the feature set chosen at the current step and $\overline{\vu}$ his complement. We define $ b : \mathcal{P}(\overline{\vu}) \to \{0, 1\}^{|\overline{\vu}|} $ a function which binarizes a sub-set of the unchosen elements. Then, given the set of selected elements $\vu$, we find the importance of the elements still not selected, while taking into account their interactions. This amounts to solving the following regression problem:  

$$
\bm{w}^t, c^t = \argmin \underset{ \bm{v} \in \mathcal{P}(\overline{\vu})}{\sum}  \big( (\bm{w}^t b(\bm{v}) + c) - v(\vu \cup \bm{v}) \big)^2
$$

The weights obtained indicate the importance of each variable by taking into account these interactions. We specify that $v(\cdot)$ is defined here as the minimization of the distance to the nearest adversarial example using the variables $\vu \cup \bm{v}$.
In the experiments, the minimization of this objective is approximated using PGD~\cite{madry2017pgd} adversarial attacks, a regression step (computation of $\bm{w^t}$) adds $10$\% of the variables and $\bm{v}$ is sampled using 1000 samples from $\mathcal{P}(\vu)$. Finally, the variables added first to get a better score.

\section{Evaluation}

For the purpose of the experiments, three fidelity metrics have been chosen. For the whole set of metrics, $\pred(\vx)$ score is the score after the softmax of the models.

\paragraph{Deletion.}~\cite{petsiuk2018rise} 
The first metric is Deletion, it consists in measuring the drop in the score when the important variables are set to a baseline state. Intuitively, a sharper drop indicates that the explanation method has well-identified the important variables for the decision. The operation is repeated on the whole image until all the pixels are at a baseline state. Formally, at step $k$, with $\vu$ the most important variables according to an attribution method, the Deletion$^{(k)}$ score is given by:

$$
\text{Deletion}^{(k)} = \pred(\vx_{[\vx_{\vu} = \vx_0]})
$$

We then measure the AUC of the Deletion scores. For all the experiments, and as recommended by~\cite{hsieh2020evaluations}, the baseline state is not fixed but is a value drawn on a uniform distribution $\vx_0 \sim \mathcal{U}(0, 1)$.

\paragraph{Insertion.}~\cite{petsiuk2018rise}
 Insertion consists in performing the inverse of Deletion, starting with an image in a baseline state and then progressively adding the most important variables. Formally, at step $k$, with $\vu$ the most important variables according to an attribution method, the Insertion$^{(k)}$ score is given by:
$$
\text{Insertion}^{(k)} = \pred(\vx_{[\vx_{\overline{\vu}} = \vx_0]})
$$
The baseline is the same as for Deletion.

\paragraph{$\mu$Fidelity}~\cite{aggregating2020} 
consists in measuring the correlation between the fall of the score when variables are put at a baseline state and the importance of these variables. Formally:
$$
\mu\text{Fidelity} = \underset{\vu \subseteq \{1, ..., d\} \atop |\vu| = k}{\operatorname{Corr}}\left( \sum_{i \in \vu} \explainer(\vx)_i  , \pred(\vx) - \pred(\vx_{[\vx_{\vu} = \vx_0]})\right)
$$

For all experiments, $k$ is equal to 20\% of the total number of variables and the baseline is the same as the one used by Deletion.

\section{Models}

The models used were all trained using Tensorflow~\cite{tensorflow2015}. For MNIST, the model is a stacking of 5 dense layers composed of (256, 128, 64, 32, 10) neurons respectively. It achieves an accuracy score above 98\% on the test set.
Concerning the Cifar-10 model, it is composed of 3 Convolutional layers of (128, 80, 64) filters, a MaxPooling (2, 2), and to Dense layer of (64, 10) neurons respectively, and achieves 75\% of accuracy on the test set. For ImageNet, we used a pre-trained VGG16~\cite{simonyan2014deep}.

\section{Targeted explanations}
\label{ap:targeted}

\begin{figure}[t!]
  \includegraphics[width=0.45\textwidth]{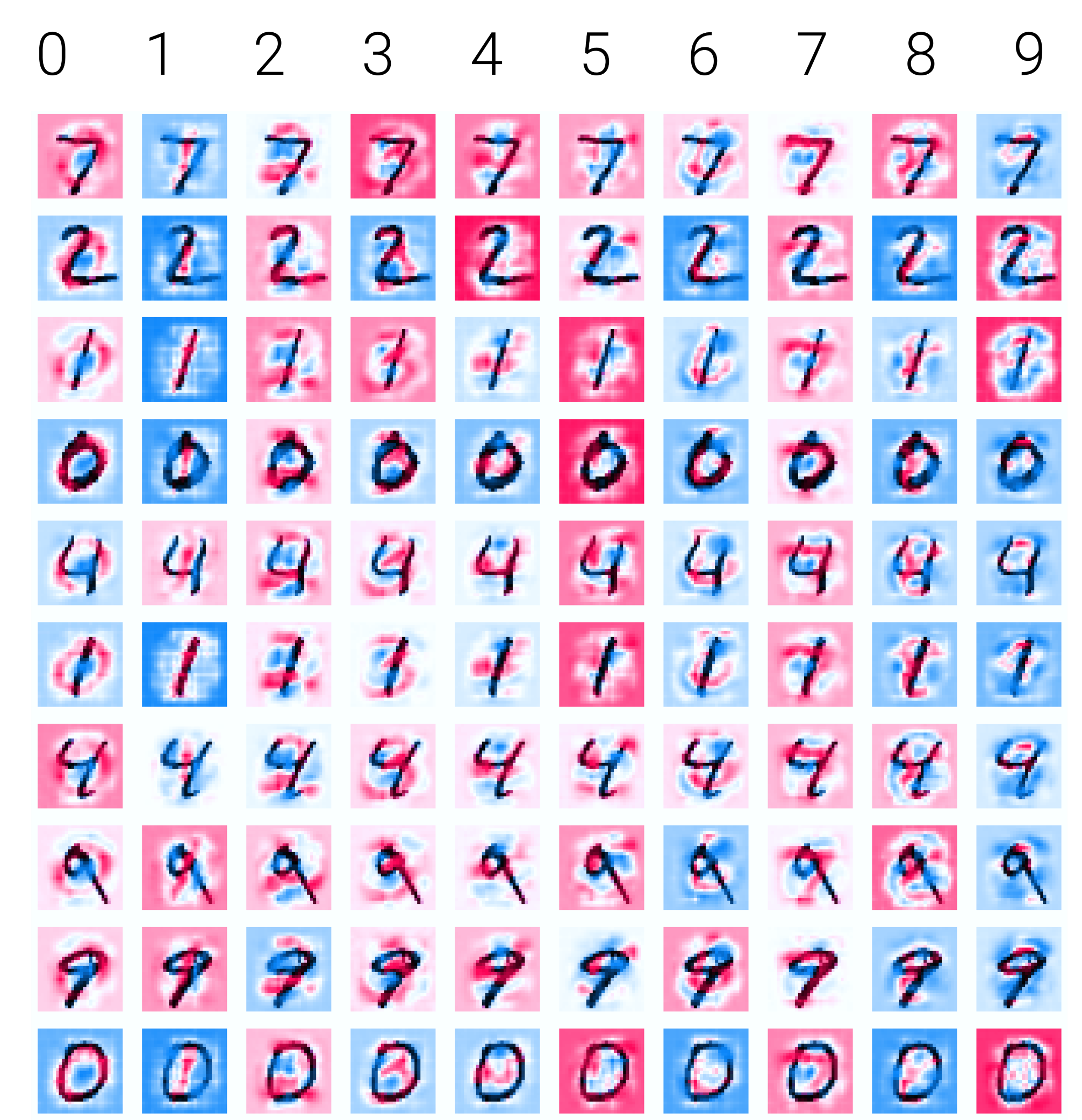}
  \caption{\textbf{Targeted Explanations} Attribution-generated explanations for a decision other than the one predicted. Each column represents the class explained, e.g., the first column looks for an explanation for the class `0' for each of the samples. As indicated in section~\ref{sec:targeted_explanations}, the red areas indicate that a black line should be added and the blue areas that it should be removed. More examples are available in the Appendix.
  }
  \label{fig:ap_targeted}
\end{figure}

In order to generate targeted explanations, we split the calls to $\eva(\cdot, \cdot)$ in two: the first one with `positive' perturbations from $\ball^{(+)}$ (only positive noise), a call with `negative' perturbations from $\ball^{(-)}$ (only negative-valued noise) as defined in Section~\ref{sec:targeted_explanations}. 

We then get two explanations, one for positive noise
$\bm{\phi}^{(+)}_{\vu} = \bm{F}_c(\mathcal{B}^{(+)}(\vx)) - \bm{F}_c(\mathcal{B}^{(+)}_{\vu}(\vx))$, the other for negative noise $\bm{\phi}^{(-)}_{\vu} = \bm{F}_c(\mathcal{B}^{(-)}(\vx)) - \bm{F}_c(\mathcal{B}^{(-)}_{\vu}(\vx))$. Intuitively, high importance for $\bm{\phi}^{(+)}_{\vu}$  means that the model is sensitive to the addition of a white line. Conversely, high importance for $\bm{\phi}^{(-)}_{\vu}$  means that removing it changes the decision model. These two explanations being opposed, we construct the final explanation as $\bm{\phi}_{\vu} = \bm{\phi}^{(+)}_{\vu} - \bm{\phi}^{(-)}_{\vu}$. More examples of results are given in Fig.~\ref{fig:ap_targeted}.

\end{document}